\documentclass{article}
\usepackage[final]{neurips_2024}
\usepackage[utf8]{inputenc}
\usepackage{xcolor}
\usepackage{newtxtext}

\usepackage{listings}
\usepackage{tabularx}
\usepackage{natbib}

\usepackage{amssymb}
\usepackage{amsthm}
\usepackage{multirow}

\usepackage{wrapfig}

\usepackage{amsmath}
\usepackage{amsthm}

\newcommand{\hadamard}{\circ}%
\newcommand{\wreath}{\wr}%

\usepackage[T1]{fontenc}

\usepackage{hyperref}

\usepackage{pslatex}
\usepackage[english]{babel}
\usepackage[utf8]{inputenc}
\usepackage{amsmath}
\usepackage{bm}
\usepackage{graphicx}
\usepackage{tikz}
\usepackage{xcolor}
\usepackage{url}
\usepackage{rotating}
\usepackage{amssymb}

\usepackage{float}
\usepackage{amsthm}
 
\newcounter{theorem}

\newtheorem{corollary}[theorem]{Corollary}

\newtheorem{defin}[theorem]{Definition}
\newtheorem{remark}[theorem]{Remark}
\newtheorem{lemma}[theorem]{Lemma}
\newtheorem{thm}[theorem]{Theorem}
\newtheorem{fact}[theorem]{Fact}

\title{The Expressive Capacity of State Space Models: A Formal Language Perspective} 

\author{Yash Sarrof, Yana Veitsman, Michael Hahn \\ Saarland Informatics Campus \\ Saarland University, Germany \\ \texttt{\{ysarrof, yanav, mhahn\}@lst.uni-saarland.de}}

\usepackage{natbib}
\usepackage{graphicx}

\begin{document}

\maketitle

\begin{abstract}
Recently, recurrent models based on linear state space models  (SSMs) have shown promising performance in language modeling (LM), competititve with transformers. However, there is little understanding of the in-principle abilities of such models, which could provide useful guidance to the search for better LM architectures. We present a comprehensive theoretical study of the capacity of such SSMs as it compares to that of transformers and traditional RNNs. We find that SSMs and transformers have overlapping but distinct strengths. In star-free state tracking, SSMs implement length-generalizing solutions to problems that transformers struggle to represent exactly. They can also model bounded hierarchical structure with optimal memory even without simulating a stack. On the other hand, we identify a design choice in current SSMs that limits their expressive power. We discuss implications for SSM and LM research, and verify results empirically \footnote{Code is available at: \href{https://github.com/lacoco-lab/ssm_expressivity}{https://github.com/lacoco-lab/ssm\_expressivity}} on a recent SSM, Mamba.
\end{abstract}

\section{Introduction}

Transformers \citep{vaswani2017attention} power most large language models (LLMs) today, as they offer the advantage of parallelized training by avoiding recurrence, compared to the previously dominant recurrent achitectures \citep[RNNs][]{elman1990finding, hochreiter1997long}.
% After their introduction \citep{vaswani2017attention}, transformers rapidly became the primary workhorse of NLP, powering most of today's large language models (LLMs).
% Compared to previously-dominant recurrent architectures \citep[RNNs][]{elman1990finding, hochreiter1997long}, transformers offered a key advantage:  parallelized training by avoiding recurrence.
However, building on a long history of continuous dynamical models \citep[e.g.][]{kalman1960general,kalman1963mathematical} and work on faster RNNs \citep{bradbury2016quasi, lei2017simple}, a recent line of work has developed \emph{state space models} (SSMs) rivaling the performance of transformers \citep[e.g.][]{gu2021efficiently, Gu2023Mamba, sun2023retentive, De2024Griffin, yang2023gated, qin2024hgrn2}.
These SSMs are recurrent models, formulated in terms of iterative state updates, while still allowing efficient parallelization.

The impressive empirical performance of such SSMs raises the question of whether they might have capabilities that the transformer architecture might lack in principle.
Simultaneously, to understand whether SSMs may plausibly overtake the dominant role of transformers, it is an important question whether SSMs may lack abilities present in transformers.
A better understanding of these questions may also point the way to future architectures that unite the strengths of both architectures.

One common approach to understanding the capabilities of computational architectures is through their expressive capacity in simulating automata and  modeling language classes; indeed, a sizeable literature has studied transformers \citep[e.g.][]{perezturing, hahn2020theoretical, bhattamishra2020ability, Yao_2021, liu2022transformers, Liu2023FlipFlop, deletang2022neural, strobl2023survey, chiang2023tighter, sanford2024representational, peng2024limitations} and  RNNs \citep[e.g.][]{siegelman1991neural, horne1993bounds, indyk1995optimal, weiss2018practical, hewitt2020rnns} through this lens.
As the difficulty of many computational problems is well-understood in terms of such language classes, results about expressive capacity directly yield results about the ability to model specific computational problems.

While a substantial number of results have been obtained for transformers and traditional RNNs, understanding remains largely open for SSMs.
In an initial step, \citet{merrill2024illusion} showed that all problems computable by SSMs are contained in $\operatorname{TC}^0$, a circuit complexity class that is known to also cover transformers \citep{merrill2023logic, strobl2023averagehard}.
Under standard conjectures, this suggests that certain types of state tracking are hard for both models.
\citet{jelassi2024repeat} and \citet{bhattamishra2024separations} provided evidence of differences between these architectures, showing that transformers outperform SSMs on copying or retrieving from long strings--tasks well within $\operatorname{TC}^0$. %In parallel,  %explored  limitations of single-layer SSMs in efficiently performing function composition over large domains. They also 
\citet{zubic2024limits} showed that multi-layer SSMs are constrained by their logarithmic space computational capacity, limiting their ability at algorithmic tasks such as multi-digit multiplication. % when addressing algorithmic tasks such as multi-digit multiplication, dynamic programming, and thus there is a significant degradation of performance despite advanced prompting techniques. 

However, a more fine-grained understanding of the power of SSMs, and how they compare to RNNs and transformers, remains an open question.
%
% \citet{jelassi2024repeat} provided evidence for differences between the architectures, showing that transformers are better than SSMs at the specific problem of copying strings -- a problem well within $\operatorname{TC}^0$. 
%
% However, beyond these results, broader detailed understanding of the power of SSMs and how they compare to RNNs and transformers remains open.
%
%
Our contribution in this paper is to provide rigorous understanding of SSMs' abilities in different classes of languages.
We show that transformers and SSMs cover overlapping but distinct fragments of $\operatorname{TC}^0$.
For instance, SSMs can model bounded hierarchical structure in ways similar to transformers and traditional RNNs, even without embedding a stack-like structure (Theorem~\ref{thm:bounded-dyck}).
For regular languages involving modular counting, such as the PARITY function (Theorem~\ref{thm:parity}), we identify a design choice that makes extant SSMs struggle in ways similar to transformers. 
In other cases, we show that SSMs resolve a failure case of transformers: they effortlessly model Flip Flop state tracking (Theorem~\ref{thm:flip_flop}).
We discuss take-aways for SSM and LLM research in Section~\ref{sec:discussion}; among others, our results suggest future LM architectures might need to combine both attention and state spaces.

\section{Background}

\subsection{State Space Models}\label{sec:background}

\paragraph{SSM Layers}
We define a single layer of a state space model as a map, at input length $T$,
\begin{align*}
\mathbb{R}^{T\times d} \rightarrow \mathbb{R}^{T \times d} & &
    (x_t)_{t=1,\dots,T} \mapsto (z_t)_{t=1,\dots,T}
\end{align*}
given by the recurrence
\begin{align}\label{eq:recurrence-ssm}
    h_t =& A(x_t) \hadamard h_{t-1} + B(x_t) &
    z_t =& \phi(h_t, x_t) 
\end{align}
where $\hadamard$ denotes elementwise product, and, for each $x_t \in \mathbb{R}^d$,
\begin{align*}
h_0 \in& \mathbb{R}^d  & & B(x_t) \in \mathbb{R}^d\ \text{(increment)}\\
    A(x_t) \in& \mathbb{R}^{d}\  \text{(gate)} & & \phi : \mathbb{R}^{2d} \rightarrow \mathbb{R}^{d}\ \text{(transform)} 
\end{align*}
We allow $A, B$ to be arbitrary smooth maps.
The map $\phi(h_t, x_t)$ includes a cascade of channel-mixing transformations and normalization, which we abstract as follows:
\begin{equation}\label{eq:mix}
    \phi(h_t, x_t) = \operatorname{Mix}_1(\operatorname{Norm}(\operatorname{Mix}_2(h_t, x_t)), x_t)
\end{equation}
where $\operatorname{Mix}_j(\cdot)$ can contain linear or (Swi)GLU components \citep[e.g.][]{qin2024hgrn2,Gu2023Mamba}.
We will take $\operatorname{Norm}$ to implement RMSNorm \cite{zhang2019root}; LayerNorm \cite{ba2016layer} can be covered by absorbing centering into $\operatorname{Mix}_2$.

\paragraph{A Full SSM}
Real-world SSMs typically stack several layers of the form (\ref{eq:recurrence-ssm}--\ref{eq:mix}).
Where needed, we use superscripts to indicate the layers in an SSM: $h_t^{(1)}, \dots, h_t^{(L)}$, where $L$ is the number of layers.
We consider input words ${\bf w} = w_{1\dots |w|}$ over a discrete alphabet $\Sigma$, and assume an encoding in terms of token embeddings $e(\sigma) \in \mathbb{R}^d$, for $\sigma\in\Sigma$. We will also write $e_\sigma$ for $e(\sigma)$. 
These feed into the lowest layer as $x_t^{(1)} := e(w_t)$.
The outputs of each layer feed into the next layer, as $x_t^{(l+1)} = z_t^{(l)}$.
The transformations in (\ref{eq:recurrence-ssm}) are specific to each layer: $A^{(1)}, \dots, A^{(L)}$ and similarly for $B, \phi$.
To keep notation simple, we will only show the superscripts where necessary for disambiguation.
The activations $z_t^{(L)}$ at the highest layer are read out by some neural network $\rho$ into vectors $q_t \in \mathbb{R}^{d_{pred}}$ describing classification or next-token predictions.
We again take $\rho$ to be an arbitrary function; importantly, all our constructions will allow $\rho$ to operate correctly even at finite precision.

\paragraph{Implementation Choices}
In Mamba, (\ref{eq:recurrence-ssm}) directly maps onto Eqs. (2a) and (2b) in \citet{Gu2023Mamba}.
The notation of \citet{Gu2023Mamba} use a matrix multiplication $\overline{A} h_{t-1}$ instead of elementwise  multiplication $A(x_t) \hadamard h_{t-1}$ in (\ref{eq:recurrence-ssm}), but importantly, Mamba's $\overline{A}$ is diagonal, so we can take $A(x_t)_i = \overline{A}_{ii}$.
Some  SSMs assume nondiagonal $A(x_t)$, but typically this matrix is diagonalizable \citep[e.g.][]{gu2021efficiently, sun2023retentive}, so that the SSM is still equivalent to one of the form (\ref{eq:recurrence-ssm}).
We discuss how other SSMs instantiate (\ref{eq:recurrence-ssm}) in Appendix~\ref{appendix:unified-ssm}.
Some models assume complex-valued activations (Appendix~\ref{appendix:unified-ssm}); our results largely do not depend on this distinction, but take it into account where needed (Theorem~\ref{thm:parity-complex}).
Some SSMs \citep[e.g.][]{Gu2023Mamba} use different numbers of channels in $x_t$ and $h_t$ using state expansion; as this does not affect expressive capacity, we will simply assume a constant dimensionality $d$.
Local convolutions \citep[e.g.][]{fu2023hungry} can be simulated with an SSM layer and  do not increase expressive capacity (Remark~\ref{app:local-co v}).

We will find that two design choices have nontrivial impact on expressive capacity:
The first one is time invariance: we call an SSM \textsc{time-invariant} if $A(x_t)$ does not depend on $x_t$.
Some SSMs, such as S4 \citep{gu2021efficiently} and Retnet \citep{sun2023retentive} are time-invariant; 
Mamba \citep{Gu2023Mamba}, Griffin \citep{De2024Griffin}, GLA \citep{yang2023gated}, HGRN \citep{qin2024hierarchically, qin2024hgrn2}, QRNN/SRU \cite{bradbury2016quasi, lei2017simple} are  not (Appendix~\ref{appendix:unified-ssm}).
The second one is the sign of the entries of $A(x_t)$:
Across all non-time-invariant SSMs surveyed, we find that the gate is always nonnegative (Appendix~\ref{appendix:unified-ssm}):
$A(x_t) \geq 0$ (\textsc{nonnegative})
due to exponential or sigmoid parameterizations of the gate -- this choice turns out to limit expressive capacity (Theorem~\ref{thm:parity}).

\paragraph{Role of Parameterization}
While the abstract form (\ref{eq:recurrence-ssm}--\ref{eq:mix}) is common across the SSM literature, differences in parameterization may have substantial effect on efficiency and training stability.
In particular, the parameterization of $A(x_t)$ has been the subject of substantial research \citep[e.g.][]{gu2020hippo,gu2021efficiently,yu2023robustifying,wang2023stablessm}.
%For instance, \citet{gu2020hippo} introduced the HiPPO framework which used polynomial projection operators to efficiently model memory across timescales, achieving faster gradient updates. S4 \citep{gu2021efficiently} built on this by decomposing $A(x_t)$ into a diagonal-plus-low-rank structure to speed up training. \citet{yu2023robustifying} improved flexibility by approximating the diagonal matrix in $A(x_t)$ while preserving HiPPO’s theoretical guarantees. \citet{wang2023stablessm} identified stable reparameterizations, such as exponential and softplus, to improve the approximation of long-term dependencies. 
However, studying expressiveness allows us to abstract away from these differences to a remarkable degree: We will allow $A, B, \rho$ to be \emph{arbitrary} functions with the given input-output properties. Our negative results are based on abstract properties of the setup (\ref{eq:recurrence-ssm}--\ref{eq:mix}), which fundamentally bottlenecks SSMs through \emph{elementwise linear} state updates.
For our positive results, will use empirical learnability experiments to verify that learnable solutions instantiating them (though not necessarily implementing the same constructions as used in the proofs) do exist in a recent SSM \citep[Mamba,][]{Gu2023Mamba}.

%\paragraph{Traditional RNNs}
We contrast SSMs with traditional RNNs such as simple RNNs or LSTMs: for these, the recurrence in Eq. (\ref{eq:recurrence-ssm}) is replaced by
$h_t = \psi(h_{t-1},x_t)$
where $\psi$ could be linear, an MLP \citep{elman1990finding}, or a more complex gated function \citep{hochreiter1997long}.

\paragraph{Finite Precision Assumption}
While Eq.(\ref{eq:recurrence-ssm}) assumes arbitrary real-valued activations, real-world implementations can only represent numbers with bounded precision.
Formally, we adopt the \emph{finite precision} notion used by \citet{weiss2018practical} in a study of the expressive power of traditional RNNs:
We allow an unbounded number of integer bits, but only $p$ fractional bits, independent of the length of the input.
See Appendix~\ref{app:precision} for discussion.

\subsection{Modeling Formal languages}\label{sec:motivation}
We study three foundational types of data structures needed for modeling formal languages \citep{hopcroft2001introduction}: finite state automata (Theorem \ref{thm:flip_flop}, \ref{thm:parity}, \ref{thm:regular}), counters (Theorem \ref{thm:counter-languages}), and  stacks (Theorem \ref{thm:bounded-dyck}). These data structures can be understood in two equivalent forms:
One is to track a state sequence over an input, where each input symbol engenders a specific transformation on the state.
The other one, more commonly considered in research on expressive capacity, considers \emph{formal languages}---sets of finite strings that are defined by the property that an automaton reaches one of a pre-specified set of ``accepting'' states after traversing the word. We focus on the latter, enabling easy comparison with existing results on transformers and RNNs.

A \textbf{finite-state-automaton} (see Definition~\ref{def:automaton}) represents a general state tracking problem over a finite state space, without imposing further structure on the state space:
The automaton keeps track of a single state from a finite state space; when reading a string from left to right, each symbol engenders a specific transformation of the state.
At each position, the current state determines which symbols can come next; membership in a formal language is determined by the state reached after reading the full string.
Finite-state-automata are equivalent in expressivity to regular expressions, and define the \textbf{regular languages} \citep{kleene1951representationof}.
% Seems like repetition, since we are going to mention this anyway, might add again, if space permits. 
% Indeed, we will be able to provide a precise criterion identifying which such finite state tracking problems--or equivalently, regular languages--SSMs such as Mamba are capable of in the finite-precision setting (Theorem~\ref{thm:regular}).

Allowing an automaton to keep track of one or more \textbf{counters} \citep{Fischer1968Counter}---integers that are incremented or decremented at each symbol read---turns the state space infinite, but in a highly structured manner.
SSMs can model this datastructure (Theorem~\ref{thm:counter-languages}), as can RNNs and transformers \citep{weiss2018practical, bhattamishra2020ability}.
\textbf{Stacks}, a first-in-first-out datastructure, enable automata to keep track of hierarchical structure, foundational to natural language \citep{lees1957syntactic}.
We show that SSMs can implement shortcut solutions to \emph{bounded} hierarchical structure even without implementing a stack (Theorem~\ref{thm:bounded-dyck}) -- these are likely to be most useful to natural language given the boundedness of human memory \citep{miller1963finitary, karlsson2007constraints}.

\begin{figure}
\begin{center}
    \includegraphics[width=0.8\textwidth]{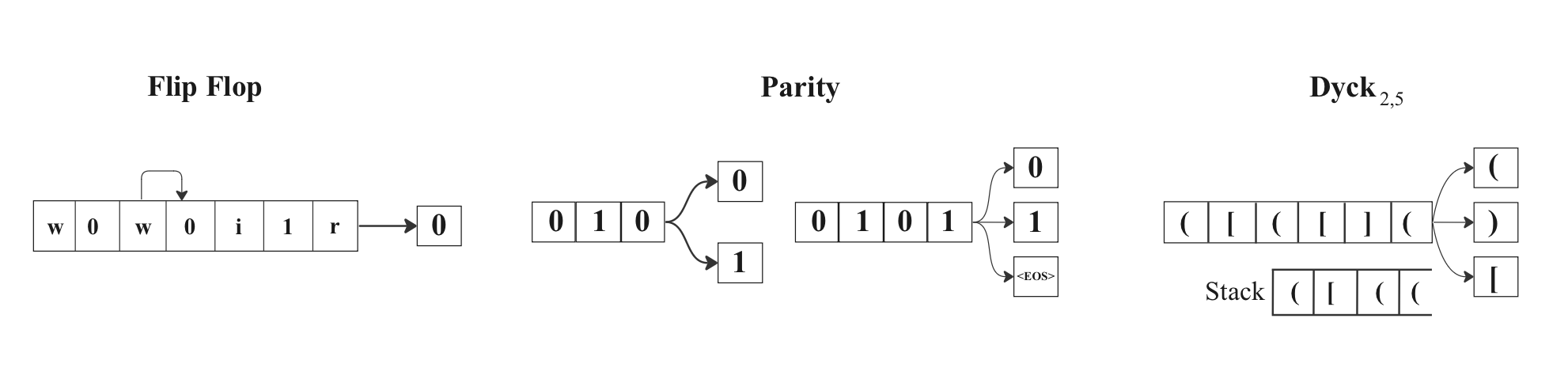}
    \end{center}
    \caption{Three key formal languages: prefixes with the sets of possible next characters: Flip Flop (Theorem~\ref{thm:flip_flop}), PARITY (Theorem~\ref{thm:parity}), bounded-depth Dyck (Theorem~\ref{thm:bounded-dyck}). In Flip Flop, after a \texttt{r} (read) instruction, the bit must match what came after the last \texttt{w} (write) instruction (here, 0). For PARITY, EOS can only follow when the number of ones in the prefix is even. For bounded-depth Dyck, a closing bracket can only appear if it matches the last unclosed opening bracket (here, ``)'' matches ``('')). Opening brackets can appear as long as the maximum depth (here, 5) hasn't been reached.}\label{fig:languages}
\end{figure}

\subsection{Formal Language Prediction and Recognition}
We fix a finite alphabet $\Sigma$.
Its elements are called \emph{characters} or \emph{symbols}.
The set of all finite strings ${\bf w}$ over $\Sigma$ is denoted $\Sigma^*$; such strings are often referred to as \emph{words}. The length of ${\bf w}$ is denoted $|{\bf w}|$.
A \emph{formal language} $L$ is a subset of $\Sigma^*$.
Techically, we assume that the alphabet includes BOS and EOS symbols, which occurs at the beginning and end of each element of $L$ and nowhere else.

We next need to define what it means for an SSM to model a formal language. The notion of \emph{recognition}, where the task is to classify a full string as belonging to the language or not.
Formally, for an SSM with $d_{pred}=1$, we say that it \textbf{recognizes} a language $L$ if the output $\rho(z_{|{\bf w}|}^{(L)})$ equals---when the SSM is run on ${\bf w} \in \Sigma^*$---1 if ${\bf w} \in L$ and 0 else.

However, such a classification task is arguably not always matched to dominant use cases in predictive sequence modeling, where the task is to predict the next token at each step.
Thus, we also cast formal languages into a language modeling and sequence prediction framework. We adopt the task of \citet{bhattamishra2020ability}, where the model is asked to output at each step in a sequence the set of possible next symbols.
Let $\operatorname{Prefix}(L) := \{w : w\in\Sigma^*, w\Sigma^* \cap L \neq \emptyset\}$ the set of valid prefixes of $L$.
We then say that a model \textbf{predictively models} a language $L$ if (Figure~\ref{fig:languages}), given a valid prefix $w \in \operatorname{Prefix}(L)$, it outputs the finite set
    \begin{equation}\label{eq:predictive-set}
    \{\sigma \in \Sigma : w\sigma\Sigma^* \cap L \neq \emptyset\}
    \end{equation}

We think of each such set as an atomic label; the set of possible labels is the power set of the finite alphabet $\Sigma$ (here, $d_{pred}=2^{|\Sigma|}$). Importantly, in both recognition and predictive modeling, we test the SSMs' ability across arbitrary input lengths, i.e. the choice of input length does not affect the inherent capability to recognize or predictively model the language. Predictive modeling can be easily converted into recognition by checking whether any symbol in the sequence is not in the predictive set at the preceding position; this can be done by adding 1 SSM layer.
Conversely, if we can show that SSMs cannot recognize a language, this proves they also cannot perform predictive modeling for it, as they then cannot correctly predict where EOS can appear.
To get the strongest results, we thus prove positive results for \emph{predictive modeling}, and negative results for \emph{recognition}.

\section{Theoretical Results}

% All of this seems repetitive, we are going to talk about all of this again, 
% We begin with  finite-state-automata -- equivalently, regular languages, a well-studied setting for understanding the expressive power of transformers \citep[e.g.][]{hahn2020theoretical, bhattamishra2020ability, angluin2023masked, liu2022transformers, Liu2023FlipFlop} and a fundamental model of state tracking when the number of possible states is finite \citep{Liu2023FlipFlop, merrill2024illusion}. Traditional RNNs can emulate all finite-state-automata at finite precision \citep{horne1993bounds, indyk1995optimal}.
% Starting with two specific languages, Flip Flop (Theorem~\ref{thm:flip_flop}) and PARITY (Theorem~\ref{thm:parity}), we derive an exact characterization of the regular languages modeled by a broad class of SSMs at finite precision (Theorem~\ref{thm:regular}).

\subsection{Length-Generalizing Representations for Flip-Flop State Tracking}\label{sec:flipFlop}

% We begin on the positive side, by establishing a case where SSMs pattern with RNNs in avoiding a failure mode of self-attention.
Flip Flop languages \citep{Liu2023FlipFlop} are a simple instance of state tracking defined in terms of \emph{write}, \emph{read}, and \emph{ignore} instructions. Each \emph{write} instruction comes with a piece of information; whenever a \emph{read} instruction is encountered, the information written by the last \emph{write} instruction is recalled.
Formally, $\mathcal{L}_{FF}$ is the set of finite strings ${\bf x}$ over $\Sigma = \{{\tt r}, {\tt w}, {\tt i}, 0, 1\}$, where $x_1, x_3, \dots \in \{{\tt r}, {\tt w}, {\tt i}\}$, $x_2, x_4, \dots \in \{0,1\}$, and where the bit following any ${\tt r}$ matches the bit following the last preceding occurrence of ${\tt w}$.
\citet{liu2022transformers} show that the Flip Flop language, as an abstraction, is a fundamental ingredient of many long-range reasoning settings.
It can be represented with a small finite-state-automaton, and LSTMs learn $\mathcal{L}_{FF}$  well \citep{Liu2023FlipFlop}.
Transformers can in principle represent it \citep{liu2022transformers, Liu2023FlipFlop}, though known constructions are not inherently length-generalizing, a fact confirmed empirically; intuitively, this may happen because  attention heads aggregate information in a commutative manner, and reliably attending to the last \emph{write} instruction requires strong position dependence in the attention weights.
SSMs, similar to traditional RNNs can easily represent Flip Flop at arbitrary input lengths and thus \textbf{avoid a failure mode of self attention}:

\begin{thm}\label{thm:flip_flop}
There is a two-layer SSM that predictively models $\mathcal{L}_{FF}$ at all lengths, at finite precision.
    %\textcolor{red}{MH: I'll try to give the theorems a more formal flavor}
\end{thm}

In the construction (Figure~\ref{fig:flip_dyck}), the first layer records the last instruction token, achieved in (\ref{eq:recurrence-ssm}) by setting $A(e({\tt r})) = A(e({\tt w})) = A(e({\tt i})) = 0$, and $A(e(0) = A(e(1)) = 1$, and setting $B(e(0)) = B(e(1)) = 0$.
Additional dimensions forward the current token to $h_{t}^{(1)}$.
In the output of the first layer $z_t^{(1)}$, whenever the input is 0 or 1, the model now has access both to the current token $w_t$ and the preceding token $w_{t-1}$, which must have been an instruction.
Based on this information, the model can set the gate to overwrite the state $h_{t-1}^{(2)}$ with the current input token when the preceding token was ${\tt w}$, and pass along the state $h_{t-1}^{(2)}$ unaltered otherwise.
This, together with  $z_t^{(1)}$, is sufficient for always identifying the legal next symbols in $\mathcal{L}_{FF}$.
The formal proof is in Appendix~\ref{sec:app:flipFlop}.

\begin{figure}
\begin{center}
    \includegraphics[width=0.8\textwidth]{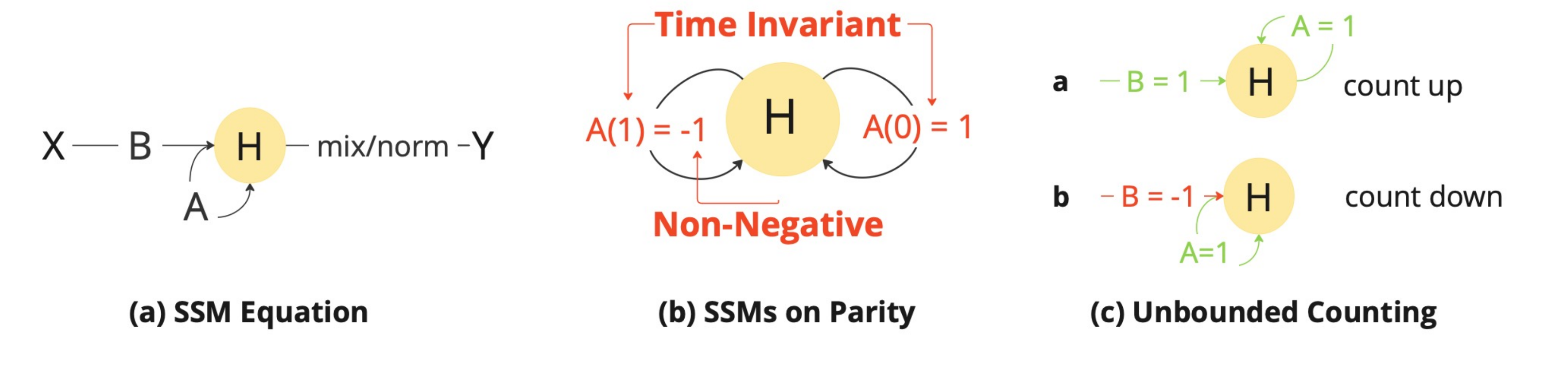}
    \end{center}
    \caption{(a) Visualizing the SSM equations \ref{eq:recurrence-ssm}, \ref{eq:mix}: The hidden state $H$ is updated by a combination of its previous values, transformed by matrix $A$, and the input $X$, modulated by matrix $B$. The updated hidden state and input are then processed through a $Mix(.)$ layer, which can incorporate components like (Swi)GLU or Linear layers, with an optional RMSNorm for normalization. (b) An intuitive construction for recognizing PARITY with SSMs is achieved by setting $B = 0$ and $A = -1$ when the input is $1$, and $A = 1$ otherwise. However, this construction violates both \textsc{nonnegative} and \textsc{time-invariant} properties. We show that one of these properties is provably required to recognize PARITY at arbitrary lengths using an SSM (Theorem~\ref{thm:parity}).
    (c) Modeling $a^nb^n$: the matrix $A$ adds the previous hidden state to the update, and depending on whether the input symbol requires counting up or down, matrix $B$ is set to $1$ or $-1$, thus making the SSM simulate a counter (Theorem~\ref{thm:counter-languages})}
    \label{fig:const_parity_flip}
    \end{figure}

\subsection{Difficulty of PARITY}

% Conversely, we next establish a design choice in SSMs which limits their power in emulating finite-state-automata, establishing -- in the finite-precision setting -- an even stronger separation between existing SSM variants and traditional RNNs than  the circuit complexity arguments in \citet{merrill2024illusion}.
PARITY, the language of bitstrings with an even number of ones, is recognized by a finite-state automaton with 2 states, and is straightforwardly encoded into a traditional RNN, even a linear one, with finite precision. It is in principle expressible for transformers \citep{chiang2022overcoming}, but is empirically hard for transformers to learn \citep{bhattamishra2020ability, deletang2022neural}, as it can provably only be represented in sharp minima \citep{HahnRofin2024Sensitive}.
A sufficiently general SSM could easily recognize it at $d=1$ by setting $h_0=1$, $A(e_1) = -1$, $A(e_0) = 0$, $B\equiv 0$, so that the sign of the single entry of $h_t$  indicates the parity (Figure~\ref{fig:const_parity_flip}).
Such an SSM would need to be non-time-invariant and require negative or complex gate values; i.e., satisfy neither \textsc{time-invariant} nor \textsc{nonnegative}. Thus, these design choices necessitated by optimization, limit the power of an SSM in emulating finite-state-automata, establishing an \textbf{even stronger separation} between existing SSM variants and traditional RNNs than the circuit complexity arguments in \citet{merrill2024illusion}
% Such properties are indeed provably needed: 

\begin{thm}\label{thm:parity}
No SSM satisfying \textsc{nonnegative}  can recognize PARITY at arbitrary input lengths with finite precision.
In particular, this applies to Mamba.
\end{thm}
The proof is in Appendix~\ref{sec:app:parity}; it examines inputs of the form $1^N$ and shows that the activations $z_N$ converge as $N\rightarrow \infty$, and thus cannot reliably encode the parity of $N$. It should be noted that we require the layer-wise operations used in the SSM to be either linear or based on the GLU or SwiGLU activation functions, as seen for instance in Mamba (Remark ~\ref{rem:activation_functions}).
As we show in Theorem~\ref{thm:parity-complex}, the same result holds even for SSMs evading \textsc{nonnegative} when they are \textsc{time-invariant}, at least when the coefficients have rational angles in the complex planes.
All extant SSMs we surveyed (Appendix, Section~\ref{appendix:unified-ssm}) satisfy either \textsc{Nonnegative} or \textsc{time-invariant}.
Hypothetical SSMs evading both \textsc{nonnegative} and \textsc{time-invariant} would be strictly stronger and can represent not only PARITY, but \emph{all} regular languages known to be in $\operatorname{TC}^0$ (Theorem~\ref{thm:solvable}).

\begin{figure}
\begin{center}
    \includegraphics[width=0.8\textwidth]{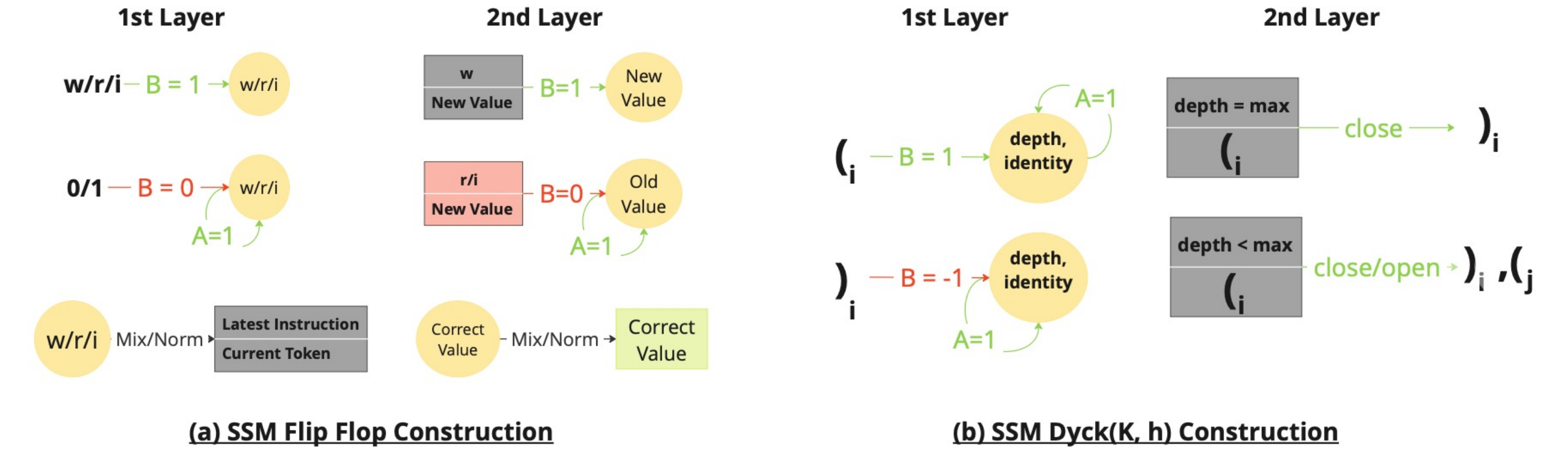}
    \end{center}
    \caption{(a) Construction for Flip-Flop (Theorem~\ref{thm:flip_flop}): The first layer stores instruction bits to the hidden state, while data bits are forwarded to the output. Hence, the output always contains both the latest instruction and the associated data bit. In the second layer, if the instruction bit is \texttt{w}, the corresponding data bit is written to the hidden state, else the old value persists. This allows the model to consistently output the correct data bit. (b) Construction for Dyck(K, h) (Theorem~\ref{thm:bounded-dyck}): The first layer tracks the depth by counting up for each opening bracket, and down for each closing bracket. The second layer builds on the Flip-Flop construction to find the last opening bracket at the current depth; the next symbol can be either the matching closing bracket or -- if the maximum depth has not been reached -- an arbitrary opening bracket.}\label{fig:flip_dyck}
\end{figure}

% \subsection{Which Regular Languages can SSMs Model?}
% This doesn't sound like a result, and doesn't go well with other definitions in the section. 
\subsection{Exact characterization of Regular Languages modeled by SSMs}

We combine Theorems~\ref{thm:flip_flop} and \ref{thm:parity} to derive an exact characterizations of the regular languages that modern non-time-invariant SSMs such as Mamba can recognize or predictively model -- the two notions coincide here -- in the finite-precision setting.
The key insight is that $\mathcal{L}_{FF}$ and PARITY are fundamental building blocks of two classes of regular languages: \emph{star-free languages} and their complement, \emph{non-star-free languages} \citep{schutzenberger1965finite, mcnaughton1971counter}:
\begin{defin}
A regular language is \emph{star-free} if it can be defined using regular expressions involving only the empty set, the empty string, individual symbols, concatenation, and Boolean combinations -- avoiding the Kleene star operation.
\end{defin}
$\mathcal{L}_{FF}$ is star-free: there is a way to define it without Kleene star.
PARITY is not star-free; any regular expression for it must involve the Kleene star.
Some languages that are intuitively defined with Kleene stars may still be star-free.\footnote{For example, $(01)^*$ is star free. It is the union of $\epsilon$ with the intersection of $0\Sigma^*$, $\Sigma^* 1$, with the complements of $\Sigma^*00\Sigma^*$ and $\Sigma^*11\Sigma^*$.}
A language is star-free if and only if it can be defined logically using only first-order quantifiers and the order relation \citep{schutzenberger1965finite}.
Also, $\mathcal{L}$ is non-star-free if and only if recognizing it involves counting modulo some finite integer $K$ \citep{mcnaughton1971counter}; % PARITY being the simplest example.
Modern non-time-invariant SSMs such as Mamba cannot perform modulo counting, but they can model \emph{all} star-free languages:
\begin{thm}\label{thm:regular}
Let $\mathcal{L}$ be a regular language.
The following are equivalent:
\begin{enumerate}
    \item There is an SSM satisfying \textsc{nonnegative} that predictively models $\mathcal{L}$ at all input lengths, at finite precision
    \item $\mathcal{L}$ is star-free.
\end{enumerate}
\end{thm}
The proof in Appendix~\ref{sec:acc:star-free} uses the Krohn-Rhodes theorem \citep{krohn1965algebraic} to reduce all star-free languages to flip flop-like state tracking.
Importantly, there are well-known constructive criteria for deciding whether a given automaton defines a star-free language \citep{schutzenberger1965finite}; hence, we have a  \emph{decidable criterion} for the finite-state tracking problems that such SSMs satisfying \textsc{Nonnegative} can solve.

This is much simpler than the situation for transformers, where an exact characterization of their power within the regular languages is complicated:
\citet{angluin2023masked} show that a certain formal abstraction of transformers (masked unique hard attention) also recognizes exactly the star-free languages, but constructions of realistic transformers via Krohn-Rhodes in \citet{liu2022transformers} do not inherently length generalize. Both theoretical \citep{huang2024formal} and empirical research indicate difficulties in generalizing even for some simple star-free languages \citep{bhattamishra2020ability, Liu2023FlipFlop}.
Known length-generalizing constructions are limited to very simple subclasses such as the piecewise testable languages \citep{yang2024counting}.
In contrast, for SSMs we have a single model per language, at finite precision and for arbitrarily long inputs.
Thus, we expect that the SSM architecture confers an advantage in star-free state tracking problems when compared to transformers -- a prediction we will find supported experimentally (Figure~\ref{fig:results-bhattamishra}).

% \subsection{Unbounded Counting}
% Again, I want the title to somehow hint towards the result, as it would help in readability.
\subsection{SSMs can perform unbounded counting}
Having characterized the regular languages modeled by SSMs, we now consider languages requiring unbounded counting \citep{Fischer1968Counter}, specifically, languages recognized by keeping track of one or more counters, where each character causes a specific increment or decrement to each counter \citep{DBLP:conf/stacs/KrebsLL15, DBLP:conf/mfcs/HahnKLL15,weiss2018practical,kutrib2021input}.
A prime example is the Dyck-1 language of well-formed strings over ``('' and ``)''; here a counter is incremented (decremented) whenever an opening (closing) bracket is encountered; a string is well-formed if and only if the counter is 0 at the end of the string.
Some other relevant formal languages are Shuffle-Dyck-$k$ (the shuffles of multiple Dyck-1 languages), $a^nb^n$ -- here, $a$ increments the counter and $b$ decrements it, and $a^nb^nc^n$ -- here, there are two counters, one keeping track of $a^nb^n$ and one of $b^nc^n$ (See Appendix~\ref{sec:language-defin-counter}).
Such counter languages are fundamental as basic context-free (Dyck-1, $a^nb^n$) or context-sensitive (e.g., $a^nb^nc^n$) languages \citep{hopcroft2001introduction}, and have been the subject of studies of both transformers \citep{bhattamishra2020ability} and RNNs \citep{weiss2018practical}.
% Many such languages are modeled by SSMs:

\begin{thm}\label{thm:counter-languages}
The languages Dyck-1, Shuffle-Dyck-$k$, n-ary Boolean Expressions, \(a^nb^n\), \(a^nb^nc^n\), and \(a^nb^nc^nd^n\), (defined in Appendix~\ref{sec:language-defin-counter}) can each be predictively modeled by an SSM.
\end{thm} 

The proof is in Appendix~\ref{sec:app:counting}.
Intuitively (Figure~\ref{fig:const_parity_flip}), an SSM can directly implement the required counters by setting $A \equiv 1$ and by defining $B(e_\sigma)$ to be the increment or decrpement cased by $\sigma$. 
In modeling such languages, SSMs pattern with both transformers \citep{bhattamishra2020ability} and LSTMs \citep{weiss2018practical}.

It may seem counterintuitive that \textsc{nonnegative} SSMs can perform unbounded counting but (by Theorem~\ref{thm:parity}) not modular counting---the latter would seem to just require reading out the value of an unbounded counter.
What is key is that, even though  $h_t$ can encode unbounded counts, reading out the modular value of an unbounded integer is a formidable problem for typical neural network nonlinearities, in particular when the information has been pushed through normalization (\ref{eq:mix}).
We should note that there is a qualitative difference between this result and the preceding positive results about finite-state languages (Theorems~\ref{thm:flip_flop} and \ref{thm:regular}), in that the construction in Theorem~\ref{thm:counter-languages} uses unboundedly large entries in the state $h_t$, whereas Theorems~\ref{thm:flip_flop} and \ref{thm:regular} use bounded values at finite precision. Indeed, we will find better length generalization in the finite-state case (Figure~\ref{fig:results-bhattamishra}).

A consequence of Theorem~\ref{thm:counter-languages} is that SSMs can recognize some languages transcending the context-free languages, as $a^nb^nc^n$ is not context-free.
A second application of the theorem, of great linguistic interest, is to bounded hierarchical structure, as we discuss next.

\subsection{Bounded Hierarchical Structure without Stacks}

It is generally agreed that hierarchical structure is a key aspect of language, and comprehending language at a human-like level requires the computational ability to process such structures 
\citep{chomsky1963algebraic, linzen2016assessing, everaert2015structures}.
The fundamental data structure for the same is a stack, where information is stored and removed as one traverses to higher and lower levels of hierarchical embedding \citep{hopcroft2001introduction}.
We now show that SSMs' counting ability can offer shortcuts on languages modeling hierarchical structure, eschewing the need for a stack.

A useful abstraction of hierarchical structure as relevant to natural language is the family of Dyck languages. The bounded-depth Dyck language $Dyck_{K,h}$ with $K$ types of parentheses and depth $h$ is the language of well-bracketed strings over $(_1$, $)_1$, $\dots$, $(_K$, $)_K$, such that the number of yet unclosed brackets never exceeds $h$ in any prefix \citep{hewitt2020rnns, yao-etal-2021-self}.
The Chomsky-Sch{\"u}tzenberger theorem \citep{chomsky1963algebraic} asserts that any context-free language can be expressed as a homomorphic image of the intersection between a Dyck language and a regular language. Specifically, the Dyck language in question refers to the classical unbounded-depth Dyck language, where $h \rightarrow \infty$, underscoring its fundamental role as the structural backbone of context-free languages. Bounding the depth reflects the fact that deep embedding is rare in natural language \citep{karlsson2007constraints, blasi2019distribution}. Prior work has found that two-layer transformers \citep{Yao_2021} and traditional RNNs \citep{hewitt2020rnns, bhattamishra2020ability} both model all $Dyck_{K,h}$ languages.
The same turns out to hold for SSMs:
\begin{thm}\label{thm:bounded-dyck}
There is a two-layer SSM with $d = \mathcal{O}(h \log K)$ that predictively models $Dyck_{K,h}$ at all input lengths, at finite precision.
\end{thm}

\begin{wrapfigure}{r}{0.4\textwidth}
    \centering
    \vspace{-2mm}
    \begin{minipage}[b]{0.32\textwidth}
    \vspace{-2mm}
        \includegraphics[width=\textwidth]{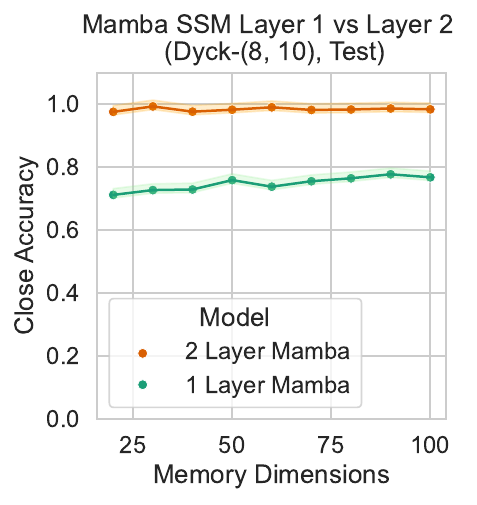}
        \label{fig:bdyck_val2}
    \end{minipage}    
    \caption{As predicted by Theorem~\ref{thm:bounded-dyck}, Mamba with 2 layers can model Dyck(K, h). Results for test set with strings of length $700 \leq n \leq 1400$.}
    \label{fig:dyck}
\end{wrapfigure}

The proof is in Appendix~\ref{app:sec:bounded-dyck}.
Intuitively (Figure~\ref{fig:flip_dyck}), the first layer records the depth of each parenthesis using the  ideas from Theorem~\ref{thm:counter-languages}, and the second layer keeps track of the last open bracket at each depth using Theorem~\ref{thm:flip_flop}.
We note that, since $Dyck_{K,h}$ is star-free, Theorem~\ref{thm:regular} already guarantees the existence of representing SSMs, but the depth and width guaranteed by Theorem~\ref{thm:bounded-dyck} is likely to be much better than what would be obtained by a black-box application of Theorem~\ref{thm:regular}:
As \citet{hewitt2020rnns} show, $h \log K$ units is optimal up to constants and is attained by traditional RNNs and LSTMs.
The SSM construction is very different from that of \citet{hewitt2020rnns} for traditional RNNs (both simple RNNs and LSTMs), which directly simulates a stack. 
Our construction is similar to the transformer construction in Theorem 4.2 in \cite{Yao_2021}, which however has to rely on specific positional encodings, unlike the SSM construction. 
This highlights that stacks are not the only way of simulating bounded hierarchical structure in recurrent architectures, and non-stack-based strategies can even attain the same optimal scaling of hidden units. 
Probing whether such stack-free shortcuts are learned by SSM-based LLMs is an exciting problem for future research.

\section{Experiments}
We have derived a fine-grained theoretical characterization of expressiveness strengths and limitations of  SSMs. We now show that our positive results can be instantiated and learned in a realistic SSM implementation, by evaluating a recent highly successful SSM, Mamba \citep{Gu2023Mamba}. 

\paragraph{FlipFlop}
We empirically instantiate  Theorem~\ref{thm:flip_flop} using the dataset of \cite{Liu2023FlipFlop},  reflecting the language $\mathcal{L}_{FF}$ as defined in Section~\ref{sec:flipFlop}.
Matching Figure 2 in \citet{Liu2023FlipFlop}, we evaluated both on in-distribution data, and on out-of-distribution data where the distance between read and write instructions tended to be larger. 
We evaluate for predicting the bits following $r$ instructions\footnote{Predictive modeling is trivial at other positions, as only the input symbols need to be considered there.}, matching the ``deterministic/clean'' mode of \citet{Liu2023FlipFlop}, and considered predictions to be correct only if all predictions within a sequence were correct. (Further details in Appendix~\ref{sec:app:setup-flipFlop}).
A small one-layer\footnote{Theorem~\ref{thm:flip_flop} constructs a \emph{two}-layer SSM. We hypothesize that Mamba uses its local convolution (Remark~\ref{app:local-co v}) to replace the lower layer from the construction in Theorem~\ref{thm:flip_flop}.} Mamba model converged to 0 error in both validation sets after $\sim$ 1400 steps (Figure~\ref{fig:flipflop}), compared to 500 steps for an LSTM reported by \citet{Liu2023FlipFlop}.  
In contrast, \citet{Liu2023FlipFlop} found that transformers kept making occasional mistakes despite training for 10K steps.

\begin{figure}
\centering
\includegraphics[width=\textwidth]{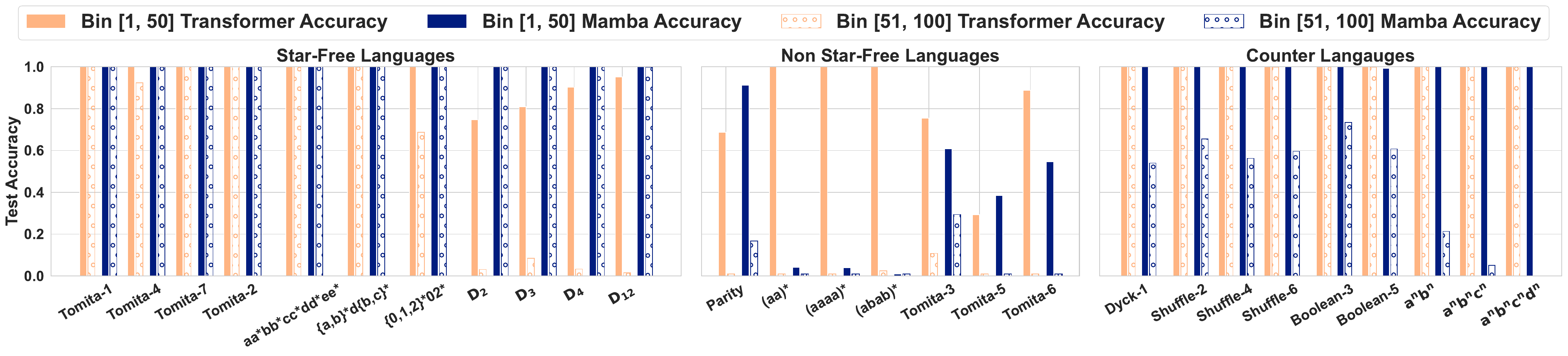}
\caption{
Results on 27 formal languages, comparing our Mamba results (blue) with transformer results reported by \cite{bhattamishra2020ability} (orange), on in-distribution lengths (solid) and out-of-distribution lengths (dotted).
As predicted by Theorem~\ref{thm:regular}, Mamba performs strongly on star-free languages, and even shows perfect length generalization.
Again as predicted by Theorem~\ref{thm:regular}, it performs poorly on non-star-free languages.
Results for transformers from \cite{bhattamishra2020ability} are mixed.
Mamba also succeeds on learning the counter languages from Theorem~\ref{thm:counter-languages}, showing perfect accuracy at in-distribution lengths at in-distribution lengths, but length generalization lags behind transformers.
}\label{fig:results-bhattamishra}
\end{figure}

\begin{wrapfigure}{r}{0.4\textwidth}
    \centering
    \vspace{-2mm}
    \begin{minipage}[b]{0.32\textwidth}
    \vspace{-2mm}
        \includegraphics[width=\textwidth]{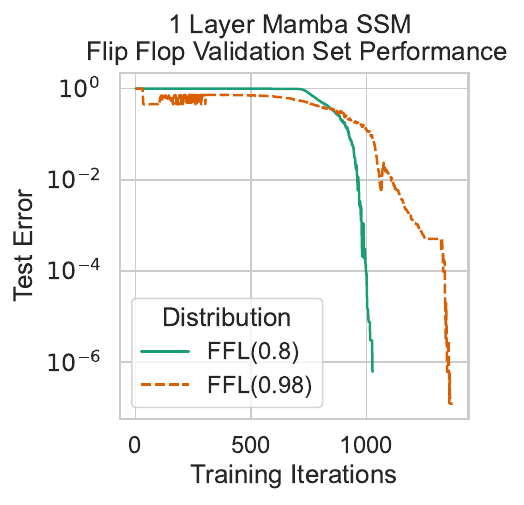}
        \label{fig:flipflop_val}
        \vspace{-2mm}
    \end{minipage}    
    \caption{Test error on the validation set for $\mathcal{L}_{FF}$, following \citet{Liu2023FlipFlop}. Mamba shows near-zero test error in both In- (green) and Out-of-distribution (orange) settings, consistent with Theorem~\ref{thm:flip_flop}, and avoids the failure seen in transformers \citep{Liu2023FlipFlop}}
    \label{fig:flipflop}
\end{wrapfigure}

\paragraph{Test Suite from \citet{bhattamishra2020ability}}
To test our theoretical results on regular and counter languages (Theorems~\ref{thm:parity}, \ref{thm:regular}, \ref{thm:counter-languages}), we test Mamba on 27 formal languages, including 18 regular languages and  9 counter languages, based on a prior study comparing transformers and RNNs \citep{bhattamishra2020ability}.
The regular languages include a popular benchmark \citep{tomita1982dynamic} and various  regular expressions; 11 are star-free.
The counter languages include the languages covered by Theorem~\ref{thm:counter-languages}.
(Definitions in Appendix \ref{sec:language-definitions}).
We  chose this test suite as it precisely covers Theorems~\ref{thm:regular} and \ref{thm:counter-languages}, and we have proven (in)expressibility results for each language in the set.

Following \cite{bhattamishra2020ability}, we trained the model for predictive modeling, i.e., at each step, the model outputs a label indicating the  set of possible next characters  (\ref{eq:predictive-set}), including EOS when required.
Following \cite{bhattamishra2020ability}, we count the model's response on an input string as correct if and only if predictive modelling was successful at \emph{all} positions in the input. 
Such a evaluation setup makes random baselines low, where a random predictor would have an accuracy exponentially small in $N$ in each of the $N$ positions. 
Training inputs have length in [1,50]; the model is evaluated on held-out bins with length [1,50] and [51,100].
Further experimental details are in Appendix~\ref{sec:app:setup-bhattamishra}.

We show our Mamba results, together with Transformer results reported by \citet{bhattamishra2020ability}, in Figure~\ref{fig:results-bhattamishra}.
LSTMs perform perfectly on all languages, and are thus not shown. In a striking confirmation of Theorem~\ref{thm:regular}, Mamba learns all star-free languages with strong length generalization but performs poorly on non-star-free languages. Transformers show more mixed results, often failing to length-generalize even on star-free languages. Consistent with Theorem~\ref{thm:counter-languages}
, Mamba, like Transformers, learns counter languages but struggles more with length generalization. The differences in Mamba's performance between star-free and counter languages may stem from the fact that the construction for the former class (Theorem~\ref{thm:regular}) is able to use finite precision and bounded state values at arbitrary input lengths, while the latter (Theorem~\ref{thm:counter-languages}) uses unbounded state values.

\paragraph{Bounded Hierarchical Structure}
To test Theorem~\ref{thm:bounded-dyck}, we recreate the experimental setup from \citet{yao-etal-2021-self}. Matching their Figure 4, we trained Mamba to predictively model $Dyck_{K,h}$ at $K=8$ and $h=10$. The training and the validation set contained samples of length $\le 700$, while the test set contained samples of length $700 \le n \le 1400$.
\citet{yao-etal-2021-self} found both transformers and LSTMs achieved strong performance on this setup.
We provide further details in Appendix~\ref{sec:experimental-dyck}.
Recall that Theorem~\ref{thm:bounded-dyck} shows that two-layer SSMs can predictively model $Dyck_{K,h}$.
We trained Mamba with 1 or 2 layers and varying dimensionality, finding that two layers can achieve essentially perfect performance across model sizes, even on the test set (Figure~\ref{fig:dyck} and~\ref{fig:three_figures}).

\section{Discussion}\label{sec:discussion}

\paragraph{Related Work}
Our work belongs to an incipient line of research into the expressiveness of SSMs \citep{jelassi2024repeat, merrill2024illusion}.
It is closely related to a long string of work studying the expressive capacity of neural sequence models, which has so far focused on recurrent networks \citep[e.g.][]{siegelman1991neural, bhattamishra2020ability, hewitt2020rnns} and, more recently, self attention \citep[e.g.][]{chiang2023tighter, merrill2023logic, strobl2023survey}.
A second link is to the classical and long-standing study of linear dynamical systems and control theory \citep{kalman1960general}.
For instance, Theorem~\ref{thm:parity} relies the asymptotic convergence of an SSM on certain inputs, establishing a link to the asymptotics of linear systems \citep[e.g.][]{phillips1992asymptotics}.

\paragraph{Take-Aways}
While theoretical in nature, our results have several actionable implications for SSM and LLM research, informing the rapidly growing research on SSM-based LLMs.
\emph{First}, encouragingly, SSMs can keep track of bounded hierarchical structure with optimal memory even without explicitly implementing a stack (Theorem~\ref{thm:bounded-dyck}), suggesting that simple diagonal linear state updates may be sufficiently powerful for modeling the hierarchical structure of language.
\emph{Second}, SSMs resolve a basic failure mode of self-attention in flip-flop state tracking while being parallellizable (Theorem~\ref{thm:flip_flop}). Overall, SSMs and attention have overlapping but distinct strengths. This lends support to the development of hybrid architectures interleaving SSM and attention layers, as instantiated very recently by Jamba \citep{lieber2024jamba}.
\emph{Third}, nonnegative gates as obtained by exponential or sigmoid parameterizations provably restrict expressive capacity, even in non-time-invariant SSMs (Theorem~\ref{thm:parity}).
While \cite{Gu2023Mamba} found no evidence that complex-valued paramerizations improved over real-valued ones in the language modality, our results suggest revisiting this question, at least for tasks where periodic state-tracking abilities may be important.
\emph{Fourth}, while exactly characterizing the capacity of transformers has proven difficult even in the finite-state case, Theorem~\ref{thm:regular} provides a decidable characterization of the regular languages -- equivalently, finite-state tracking problems -- that SSMs such as Mamba can  model. Such decidable characterizations may make it easier to theoretically predict abilities and anticipate failures of LLMs; exploring the implications of this characterization in more realistic setups is an exciting direction for future research.

\paragraph{Limitations}
The main limitation of our theoretical results is that they focus on in-principle expressiveness, and do not directly make statements about learning and generalization.
Future work could address this, for example, by examining whether our constructions result in reasonably flat minima, or by studying gradient flow dynamics.
While we empirically verified that our positive results can indeed be instantiated, in a learnable manner, in one realistic SSM implementation, implementational differences might still result in practical differences between implementations. Studying the role of such implementational differences is an interesting problem for future work; we have made a first step by theoretically elucidating the implications of nonnegative gate values.

\section{Conclusion}

We have studied the expressive capacity of modern state space models (SSMs), through the lens of automata and formal languages.
We have shown theoretically that SSMs can express star-free languages, a range of counter languages, and bounded hierarchical structure.
By providing rigorous results about the expressiveness of the SSM architecture, our results can provide guidance to work on SSM-based language models.

\section*{Acknowledgments} 
We thank Mark Rofin for useful discussion about Theorem~\ref{thm:parity} and we thank anonymous reviewers for their helpful feedback. We gratefully acknowledge the insightful discussions with members of the FLaNN community which contributed to the ideas of this project. Funded by the Deutsche Forschungsgemeinschaft (DFG, German Research Foundation) – Project-ID 232722074 – SFB 1102.

\bibliographystyle{abbrvnat}
\bibliography{literature}

\clearpage

\appendix

\section{Instantiations of General Framework in SSM Models}\label{appendix:unified-ssm}
Here, we survey how (\ref{eq:recurrence-ssm}) is instantiated in a range of SSMs.
As stated in Section~\ref{sec:background}, we refer to SSMs where the gate $A$ does not depend on $x_t$ as \emph{time-invariant}.
An equivalent terminology is the distinction between ``Weak Linear Time Invariant Convolutional Models'' (i.e., time-invariant) and ``Linear Time Variant Models'' (i.e., non-time-invariant) in \citet{akyurek2024context}.

\subsection{Non-Time-Invariant Models}

Approximately simultaneously with or more recently than \cite{Gu2023Mamba}, a range of non-time-invariant SSMs have been introduced \citep{De2024Griffin, yang2023gated, qin2024hierarchically, qin2024hgrn2}.
This category also covers highly similar earlier RNN variants \citep{bradbury2016quasi, lei2017simple}.

\paragraph{Mamba}
In Mamba, (2) and (3) directly map onto Eqs. (2a) and (2b) in \citet{Gu2023Mamba}.
The notation of \citet{Gu2023Mamba} use a matrix multiplication $\overline{A} h_{t-1}$ instead of elementwise  multiplication $A(x_t) \hadamard h_{t-1}$ in (REF), but importantly, Mamba's $\overline{A}$ is diagonal, so we can take $A(x_t)_i = \overline{A}_{ii}$.
Due to exponential parameterization, its entries are nonnegative.

\paragraph{Griffin}
The  RG-LRU layer of Griffin \citep{De2024Griffin} uses the equation
\begin{align*}
    h_t &= \underbrace{a_t}_{A(x_t)} \hadamard h_{t-1} + \underbrace{\sqrt{1-a_t^2} \hadamard (i_t \hadamard x_t)}_{B(x_t)}
\end{align*}
where $a_t, i_t$ are neurally parameterized in terms of $x_t$ but not $h_{<t}$; by design, $a_t \in (0,1)$.
$\phi$ is instantiated in terms of linear transformations, GeLU, and RMSNorm (Figure 2 in \citet{De2024Griffin}).
The local attention used by Griffin can be subsumed into an SSM layer (Remark~\ref{app:local-co v}).

\paragraph{Gated Linear Attention  \citep[GLA][]{yang2023gated}}
This model (Section 4.4 in \citet{yang2023gated}) instantiates our framework using a recurrence of the form (\ref{eq:recurrence-ssm}); while the state is two-dimensional in this model, the update is performed by elementwise products as in (\ref{eq:recurrence-ssm}).
The gate is obtained by applying sigmoid to a linear transformation of $x_t$; thus, its entries are in $(0,1)$.
$\phi$ is instantiated in terms of SwiGLU and LayerNorm.

\paragraph{HGRN}
HGRN \citep{qin2024hierarchically} and HGRN2 \citep{qin2024hgrn2} are defined by a recurrence of the form (\ref{eq:recurrence-ssm}); the gate entries are $\in (0,1)$ by design.
$\phi$ is instantiated in terms of GLU, linear transformations, and normalization.
In HGRN, the state is complex, but crucially the gate remains real-valued.

\subsection{Time-Invariant Models}
Time-invariant SSMs introduced before late 2023 are surveyed by \citet[][Appendix B]{Gu2023Mamba}, such as \citep{mehta2022long,sun2023retentive,orvieto2023resurrecting}.
Time-invariant SSMs have often used complex-valued states and gates; this does not have a major impact on our results: First, as complex-valued SSMs subsume real-valued ones, our positive results carry over. Second, our negative result about PARITY is affected by this distinction and requires a separate argument, see Theorem~\ref{thm:parity-complex}.

Note also that $A h_{t-1}$ is often described as a general matrix multiplication, but $A$ is diagonalizable (e.g. Lemma 3.2 in \citet{gu2021efficiently}; \cite{sun2023retentive} for RetNet), which ---even though implementation may be based on non-diagonalized representations \citep{gu2021efficiently}---renders the model equivalent to one where $A$ is diagonal from the start. This equivalence is shown as Lemma 3.1 in \citet{gu2021efficiently}.

\section{Formal Definitions and Proofs}

\subsection{Flip Flop}\label{sec:app:flipFlop}

We begin by introducing key notions of automata theory.
References for automata theory include \citet{eilenberg1974automata, hopcroft2001introduction, sakarovitch2009elements}.
We will provide those key notions that are necessary to prove our results.
We will focus on \emph{deterministic} finite-state-automata (DFA), and simply refer to them as \emph{finite-state-automata}.\footnote{A closely related notion is the \emph{semiautomaton}, which is the notion considered in the closely related work \citet{liu2022transformers}. Semiautomata lack a fixed start state $q_0$. We include $q_0$, but this difference is not substantial for our formal results.}
First,
\begin{defin}\label{def:automaton}
    A (deterministic) finite-state-automaton $\mathcal{A}$ consists of:
    \begin{itemize}
        \item a finite alphabet $\Sigma$
        \item a finite state set $Q$
        \item a starting state $q_0 \in Q$
        \item a transition function $u : Q \times \Sigma \rightarrow Q$
    \end{itemize}
We extend $u$ to a map $u : Q \times \Sigma^* \rightarrow Q$ by setting:
    \begin{align*}
    u(q, \epsilon) &= q \\
        u(q, w_{1\dots i+1}) &= u(u(q, w_{1\dots i}), w_{i+1})
    \end{align*}
    where $\epsilon$ is the empty word.
    
    Intuitively, $u(q_0, {\bf w})$ is the state that $\mathcal{A}$ is in after reading ${\bf w}$.

   The automaton \emph{recognizes} a language $L \subseteq \Sigma^*$ if there is a \emph{recognizing set} $R \subseteq Q$ such that
   \begin{equation}\label{eq:recog-set}
   L := \{w : u(q_0, {\bf w}) \in R\}
   \end{equation}
\end{defin}
Kleene's Theorem \citep{kleene1951representationof} asserts that a language $L \subseteq \Sigma^*$ is regular (i.e., defined by a regular expression) if and only if it is recognized by some finite-state automaton.

A very fundamental automaton underlying Flip Flop is:
\begin{defin}\label{def:set-reset}
    A set-reset automaton is a finite-state-automaton where $(Q $ \textbackslash $\{q_0\}) \subseteq \Sigma$ and
    \begin{equation}
        u(q,\sigma) = \begin{cases}q & \text{ if }\sigma \not\in Q \\
        \sigma & \text{ else}\end{cases}
    \end{equation}
\end{defin}
Intuitively, such an automaton keeps recording the last seen symbol from a designated set $Q \subseteq \Sigma$.
Such an automaton is easily simulated with a single non-time-invariant SSM layer:
\begin{lemma}\label{thm:set-reset}
Let $\mathcal{A} = \langle \Sigma, Q, q_0, u\rangle$ by a set-reset automaton.
Then there is a single-layer SSM with finite precision and width $d=1+\log Q$ that maps each $w_{1\dots T} \in \Sigma^*$ to the state sequence $u(q_0, w_1), u(q_0, w_{12}),\dots, u(q_0, w_{1\dots T}) \in Q^T$. 
    
    Formally, there is an injective map $V : Q \rightarrow \mathbb{R}^d$ such that $\rho(z_t)= V(u(q_0, w_{1\dots t}))$ for $t=1,\dots,T$.
\end{lemma}

\begin{proof}
Let $B(\sigma) \in \mathbb{R}^{\log |Q|}$ be a binary encoding if $\sigma\in Q$, and ${\bf 0}  \in \mathbb{R}^{\log |Q|}$ else.
Take $h_0 = B(q_0)$.
Let $A(\sigma) = {\bf 0}$ if $\sigma\in Q$ and $A(\sigma) = {\bf 1}$ else.
After processing a string, the state $h_t$ is $B(\sigma)$ where $\sigma$ is the last symbol in $Q$ that has occurred if any has, and $B(q_0)$ otherwise.
Coming to (\ref{eq:mix}, in order to avoid division by zero when normalizing if no element of $Q$ has been read, we add a dummy dimension to $h_t$ whose value is always $1$.
We take $\operatorname{Mix}_1, \operatorname{Mix}_2$ to be the identity.
Note that, even though normalization will affect the numerical values,  the binary encoding of $\sigma\in Q$ can still be read out with finite precision, as  $1 \leq \|h_t\|_2 \leq \sqrt{1+\log |Q|}$, and thus nonzero entries will remain bounded away from zero.
\end{proof}

\begin{thm}[Restated from Theorem~\ref{thm:flip_flop}]\label{thm:flip_flop_app}
    There is a two-layer SSM that predictively models $\mathcal{L}_{FF}$ at all lengths, at finite precision.
\end{thm}

\begin{proof}
    In the first layer, we use Lemma~\ref{thm:set-reset} to simulate a set-reset automaton over the input alphabet $\Sigma_1 = \{w,r,i,0,1\}$ where $Q_1 = \Sigma_1 \cup \{q_0\}$.
    This layer outputs at each position whether the last instruction was write, read, or ignore.
    The layer additionally, at each position, forwards the input symbol using additional dimensions.
    Formally, at the first layer, $\rho(h_t)$ allows us to read out the input symbols $x_{t-1}, x_t \in \Sigma$.
    
    In the second layer, we again use Lemma~\ref{thm:set-reset} to simulate a set-reset automaton over an extended alphabet $\Sigma_2 := \Sigma_1 \times \Sigma_1$, where the first component indicates the input symbol $x_t$ and where the second component indicates $x_{t-1}$.
    In this set-reset automaton, $Q_2$ contains, besides a start state $q_0$, those elements of $\Sigma_2$ whose second entry is $w$.
    The second layer thus keeps track of the input bit $b \in \{0,1\}$ following the last write instruction.
    It additionally forwards the input symbol $x_t$ using additional dimensions.

    The second layer, via $\rho$, then predicts the possible next symbols on the basis of this information:
    If $x_t \in \{0,1\}$, any instruction in $\{w,r,i\}$ is possible.
    If $x_t \in \{w,i\}$, any bit in $\{0,1\}$ is possible.
    If $x_t = r$, the bit stored after the last write instruction is possible; if no write instruction has appeared (hence, the second automaton is still in its start state), any bit in $\{0,1\}$ is possible.
\end{proof}

\subsection{Difficulty of Representing PARITY}\label{sec:app:parity}

\begin{defin}
    PARITY is the regular language over $\Sigma=\{0,1\}$ of strings where the number of ones is even. As a regular expression, PARITY is $(0^*10^*10^*)^*$.
\end{defin}

\begin{thm}[Restated from Theorem~\ref{thm:parity}]\label{thm:parity_app}
No SSM satisfying \textsc{nonnegative}  can recognize PARITY at arbitrary input lengths with finite precision.
\end{thm}

\begin{proof}
  
We consider an SSM with multiple layers, and indicate the layer in superscript: $h_t^{(1)}, \dots, h_t^{(L)}$.
We write $z_t^{(0)}$ for the input token embedding $e(w_t)$.
Consider a SSM processing the word $1^t$, for $t \rightarrow \infty$.
We show, by induction over the number of layers, the following claim:

\textit{($\dagger$) Each entry of $z_t^{(k)}$ converges to a value bounded, in absolute value, by a constant.}

By the assumption of finite precision, convergence automatically leads to the entries becoming ultimately constant.
Once we have shown this, we know that $z_t^{(L)}$ is constant when $t$ is sufficiently large; thus, the parity of the string $1^t$ cannot be read out from $z_t^{(L)}$. As a consequence, the SSM cannot recognize PARITY. Indeed, we have shwon the stronger claim that the language $(11)^*$ -- the language of even-length strings over one symbol -- is not recognized by an SSM; we will use this stronger statement in Corollary~\ref{eq:nonnegative-lower-bound-periodic}.

We proceed to proving ($\dagger$).
The claim ($\dagger$) is trivially true at $k=0$, as the input token is always the same and we defined $z_t^{(0)} := e(w_t)$.
Now consider $k>0$.    
    By hypothesis, the activations are given as
    \begin{equation}
        h_t^{(k)} = A(x_t) \hadamard h_{t-1}^{(k)} + B(x_t) \\
    \end{equation}
where $A(x_t), B(x_t)$ are constant $\alpha := A(x_t)$, $\beta := B(x_t)$ when $t>T_0$, for some $T_0 > 0$.
The solution of the recurrence for $t>T_0$ is
\begin{equation}
    h_t = \alpha^{t-T_0} \left(h_{T_0} + \frac{\beta}{\alpha-1}\right) + \frac{\beta}{1-\alpha}
\end{equation}
Each dimension $j=1,\dots,d$ of this vector can be constant (if $(h_{T_0})_j + \frac{\beta_j}{\alpha_j-1}=0$), diverge exponentially ($\alpha_j>1$), converge exponentially ($\alpha_j<1$) or diverge linearly ($\alpha_j=1$).

We next need to show that $z_t = \operatorname{Mix}_2(\operatorname{Norm}(\operatorname{Mix}_1(h_t,x_t)))$ converges.

First, consider the effect of applying a linear transformation to the state $h_t$.
Each entry of the result will be some linear combination
\begin{equation}
u_t = \lambda_1 (h_t)_1 + \dots  + \lambda_d (h_t)_d
\end{equation}
If each $\alpha_j < 1$, then $u_t$ converges.
If some $|\alpha_j| \geq 1$, there may be some cancellation if $\alpha_i=\alpha_j$ for some $i\neq j$; cancellation can only lead to full erasure of the relevant terms or to a remaining term with the same exponent.
In conclusion, each entry $u_t$ will again either converge to a finite value or diverge towards $\pm \infty$.

We now need to understand the behavior of $\operatorname{Mix}_1(h_t,x_t)$.
Recall that, based on our survey (Appendix~\ref{appendix:unified-ssm}), we allowed it to contain linear, GLU \citep{dauphin2017language}, and SwiGLU \citep{shazeer2020glu} components.
If $\operatorname{Mix}_1(h_t,x_t)$ implements a linear transformation only, each entry likewise may converge, diverge linearly, or diverge exponentially.
We note that---if $\sigma$ is the sigmoid function---$\sigma(u_t)$ always converges, as $\sigma$ simply saturates to 0 or 1 if $u_t$ diverges.
Hence, if $\operatorname{Mix}_1(h_t,x_t)$ implements GLU, each entry likewise may converge, diverge linearly, or diverge exponentially.
Finally, if $\operatorname{Mix}_1(h_t,x_t)$ implements SwiGLU, each entry of the result will be a product of a linear combination of the form $u_t$, and $Swish_\beta$ applied to another such linear combination. Depending on the behavior of these two $u_t$-like terms, the outcome will behave as a product of sequences that may converge exponentially, diverge exponentially, or diverge linearly -- e.g., the outcome may also diverge quadratically, or converge as $n\alpha^{-n}$, etc.

If all dimensions of $\operatorname{Mix}_1(h_t,x_t)$ converge, then $\operatorname{Norm}(\operatorname{Mix}_1(h_t,x_t))$ will also converge to a scaled version of $\frac{\beta_i}{1-\alpha_i}$, scaled by a bounded factor as $\beta_i \neq 0$.
Now assume some dimensions of $\operatorname{Mix}_1(h_t,x_t)$  do not converge; in this case, for any two dimensions $i,j$, either their ratio will converge to a constant, or converge to 0 or $\pm \infty$.
After applying $\operatorname{Norm}(\cdot)$, the entries asymptotically dominating the others will converge to a finite value bounded, in absolute value, by 1; the others will converge to zero.

In conclusion, we have found that each entry of $\operatorname{Norm}(\operatorname{Mix}_1(h_t,x_t))$ converges to some number bounded, in absolute value, by $1$.
As $\operatorname{Mix}_2$ is continuous, each entry of $z_t$ likewise converges, with a bound depending on the Lipschitz constant of $\operatorname{Mix}_2$.
\end{proof}

We next show the result, referenced in the main paper text after Theorem~\ref{thm:parity}, about time-invariant SSMs with complex-valued gates:
\begin{thm}\label{thm:parity-complex}
   \textsc{Time-invariant} SSMs cannot recognize PARITY with finite precision at arbitrary input lengths, even with complex-valued gates, as long as each entry in each $A$ has a rational angle in the complex plane.
\end{thm}
Here, by a \emph{rational angle}, we refer to an angle that is a rational number when expressed in degrees; such angles are rational multiples of $2\pi$ when expressed in radians.
As the rational angles are dense in the reals, one expects that even if some irrational angles permitted modeling PARITY, such solutions would be very hard to find -- in particular given that irrational numbers are not exactly represented in finite precision.

\begin{proof}
By assumption, any $A_j \in \mathbb{C}$ in any layer can be written as
\begin{equation}
    A_j = r_j \exp(2\pi i q_j)
\end{equation}
where $q_j \in [0,1]$ is rational and $r_j \geq 0$ is real -- here, $2\pi q_j$ is known as the \emph{argument} of $A_j$; it describes the angle of $A_j$ in the complex plane in radians. Correspondingly, the angle in degrees is described by $q_j \cdot 360^\circ$; this is rational if and only if $q_j$ is.

As a time-invariant SSM has a finite number of such values $A_j$, across all its layers, we can select a positive integer $W$ such that $Wq_j \in \mathbb{N}$ for each $j$, in each layer.
Importantly, $(A_j)^W = (r_j)^W \in \mathbb{R}$.

Now consider the action of any layer of the SSM on an input sequence of the form $A_T = (10^{W-1})^T$.

The claim is that, for each $i=1,\dots, W$, the sequence
\begin{equation}
    z_{tW+i}^{(k)}
\end{equation}
converges as $t\rightarrow \infty$.
As in the proof of Theorem~\ref{thm:parity}, in the finite-precision setting, converge entails that the sequence becomes ultimately stationary. Note that the parity of $A_T$ equals the parity of $T$; hence, it is impossible to read out the parity from $z_{TW}^{(k)}$ when $T$ is large.

Now consider, suppressing the index for the dimension in $1,\dots, d$:
\begin{align*}
    h_{tW}^{(k)} &= \sum_{i=1}^{tW} A^{tW-i} B(z_{i}^{(k-1)}) \\
    &= \sum_{i=1}^{tW} A^{tW-i} B(z_{i}^{(k-1)}) \\
    &= \sum_{s=1}^{t} \sum_{j=sW}^{(s+1)W-1} A^{tW-j} B(z_{sW+j}^{(k-1)}) \\
    &= \sum_{s=1}^{t} \sum_{j=0}^{W-1} A^{(t-s)W-j} B(z_{sW+j}^{(k-1)}) \\
    &= \sum_{s=1}^{t} \sum_{j=0}^{W-1} (r \exp(2\pi i q))^{(t-s)W-j} B(z_{sW+j}^{(k-1)}) \\
    &= \sum_{s=1}^{t} \sum_{j=0}^{W-1} r^{(t-s)W-j} \exp(-2\pi i jq) B(z_{sW+j}^{(k-1)}) \\
    &=\sum_{j=0}^{W-1}\exp(-2\pi i jq) \sum_{s=1}^{t}  r^{(t-s)W-j}  B(z_{sW+j}^{(k-1)}) \\
\end{align*}
Separately considering summation beyond $T_0$ at which $z_{tW+j}^{(k-1)}$ has become stationary, we get
\begin{align*}
    &= \underbrace{\left[\sum_{j-0}^{T_0-1} \dots\right]}_{U_1} + \left[\underbrace{\left(\sum_{j=T_0}^{W-1}\exp(-2\pi i jq) B(z_{j}^{(k-1)}) r^{-j}\right)}_{U_2} \underbrace{\left(\sum_{s=1}^{t}  r^{(t-s)W} \right)}_{U_3}  \right] \\
\end{align*}
$U_1$ and $U_2$ do not depend on $t$.
Intuitively, $U_2 \in \mathbb{C}$ determines a direction in the complex plane, whereas $U_3 \in \mathbb{R}$ determines a magnitude.
It remains to understand $U_3$, which can be rewritten as:
\begin{equation}
    U_3 = \sum_{s=1}^t r^{(s-1)W} = r^{-W} \sum_{s=1}^t (r^{W})^s = r^{-W} \sum_{s=0}^t (r^{W})^s - r^{-W}  = r^{-W} \begin{cases}  \frac{1-(r^{W})^t}{1-(r^{W})}-1 & r \neq 1 \\ s-1 & r = 1\end{cases} 
\end{equation}
We have now achieved a situation like in the proof of Theorem~\ref{thm:parity}: $U_3$ can converge exponentially, diverge linearly, or diverge exponentially. The remainder of the proof is analogous to that proof.
\end{proof}

The following Corollary of Theorem~\ref{thm:parity} will be used in the proof of Theorem~\ref{thm:regular}:
\begin{corollary}\label{eq:nonnegative-lower-bound-periodic}
Assume \textsc{nonnegative}, SSMs with finite precision cannot recognize any non-star-free regular language.
\end{corollary}

\begin{proof}
For any non-star-free regular language $\mathcal{L}$, there are words $u,v,w$ such that the membership $uv^nw \in \mathcal{L}$ is determined by the value of $n$ modulo some finite integer $k$ (depending on $\mathcal{L}$) \citep{mcnaughton1971counter}. 
Fix any such $u, v, w \in \Sigma^*$.

Now assume an SSM satisfying \textsc{nonnegative} can recognize  $\mathcal{L}$ with finite precision.
We can subsume the action of $u$ into the state $h_0$ by taking $h_0$, in each layer, to be the state of the SSM after reading $u$.
We now have an SSM that can determine the parity of $t$ when fed a word of the form $v^tw$.

For this SSM, we want to show

\textit{($\dagger$) When fed words of the form $v, v^2, v^3, \dots$, for each $i=0, \dots, |v|-1$, and each layer $k=1,\dots,L$, the sequence $z_{t|v|+i}^{(k)}$ converges as $t \rightarrow \infty$.}

As in the preceding two proofs in this section, convergence entails becoming ultimately constant in the finite-precision setting. 

The claim ($\dagger$) is immediate at $k=0$.

Now at $k>0$, we write:
\begin{align*}
    h_{t|v|+i}^{(k)} =& A(z_{t|v|+i}^{(k-1)}) \dots A(z_{(t-1)|v|+i+1}^{(k-1)}) h_{(t-1)|v|+i}^{(k)} \\
& + A(z_{t|v|+i}^{(k-1)}) \dots A(z_{(t-1)|v|+i+2}^{(k-1)}) B(z_{(t-1)|v|+i+1}^{(k-1)}) \\
& + A(z_{t|v|+i}^{(k-1)}) \dots A(z_{(t-1)|v|+i+3}^{(k-1)}) B(z_{(t-1)|v|+i+2}^{(k-1)}) \\
& + \dots \\
& + B(z_{(t-1)|v|+i}^{(k-1)}) \\    
\end{align*}
On the RHS, as $t\rightarrow 0$, all terms except for $h_{(t-1)|v|+i}^{(k)}$ become constant by the inductive hypothesis.
Hence, there are some $\alpha, \beta$ such that, for sufficiently large $t$,
\begin{equation}
        h_{t|v|+i}^{(k)} = \alpha \hadamard h_{(t-1)|v|+i}^{(k)} + \beta
\end{equation}
We are now, for each $i$, in the same situation as in the proof of Theorem~\ref{thm:parity}]: each dimension of this recurrence can converge exponentially, diverge exponentially, or diverge linearly; as in that proof, it follows that  $z_{t|v|+i}^{(k)}$ converges as $t\rightarrow \infty$.

We have shown ($\dagger$).

We now follow up by showing that

\textit{($*$) When fed words of the form $v^tw$, for each $i=1, \dots, |w|$, and each layer $k=1,\dots,L$, the sequence $z_{t|v|+i}^{(k)}$ converges as $t \rightarrow \infty$.}

Again, at finite precision, convergence entails that the sequences are ultimately constant.
Again, ($*$) is true at $k=0$ trivially.
When feeding the SSM words of the form $v^tw$, in each layer, the final state is in each layer $k$, at each $i=1, \dots, |w|$:
\begin{align*}
    h_{t|v|+i}^{(k)} =& A(z_{t|v|+|w|}^{(k-1)}) \dots A(z_{t|v|+1}^{(k-1)}) h_{t|v|}^{(k)} \\
    & + A(z_{t|v|+i}^{(k-1)}) \dots A(z_{t|v|+2}^{(k-1)}) B(z_{t|v|+1}^{(k-1)}) \\
    & + \dots \\
    & + B(z_{t|v|+i}^{(k-1)}) \\
\end{align*}
By inductive hypothesis, for large $t$, there are $\psi_i, \gamma_i$ such that
\begin{align*}
    h_{t|v|+i}^{(k)} =& \psi_i \hadamard h_{t|v|}^{(k)} + \gamma_i
\end{align*}
and, as shown before, each entry of $h_{t|v|}^{(k)}$ converges exponentially, diverges exponentially, or diverges linearly.
Now, by assumption, one can read out, at finite precisiion, the parity of $t$ from
\begin{align*}
    z_{t|v|+|w|}^{(L)} = \operatorname{Mix}_1(\operatorname{Norm}(\operatorname{Mix}_2(\psi_{|w|} \hadamard h_{t|v|+i}^{(k)} + \gamma_{|w|})))
\end{align*}
We now simply absorb the operation $X\mapsto \psi_{|w|} \hadamard X + \gamma_{|w|}$ into $\operatorname{Mix}_2$, and obtain by the same arguments as in the proof of Theorem~\ref{thm:parity} that $z_{t|v|+|w|}^{(L)}$ converges as $r\rightarrow\infty$. This is a contradiction to the claim that the value  of $t$ can be read out, modulo $k$, from $z_{t|v|+|w|}^{(L)}$ at finite precision.
\end{proof}

\begin{remark} \label{rem:activation_functions}
    As outlined in our analysis, the assumptions in Theorem \ref{thm:parity} are based on layer-wise operations that are either linear or based on the GLU or SwiGLU activation functions. This assumption is critical to the proof: one could design activation functions that make PARITY expressible.

Given a sequence $\mathbf{x} = x_1, \dots, x_T$, consider the function $f(\mathbf{x}) = \frac{e^{i \pi \sum_{i=1}^n x_i} + 1}{2}$. This continuous function is designed to satisfy the condition that, for bit-strings $\mathbf{x}$, $f(\mathbf{x}) = 1$ if $\sum_{i=1}^n x_i$ is even, and $f(\mathbf{x}) = 0$ otherwise.
At first glance, it seems like this function can be approximated by a cumulative sum layer in combination with a two-layer SSM to compute $f(x) = \frac{e^{i \pi x} + 1}{2}$.

However, this construction cannot be implemented under the condition for which we prove Theorem \ref{thm:parity}. This is because computing this function $f(x)$ inherently requires a layer-wise nonlinear operation (such as a MLP) capable of representing sine and cosine functions over arbitrarily large input values. Importantly, achieving a construction that works for any input length requires the ability to handle arbitrarily large inputs within a single operation.

A single GLU or SwiGLU activation function, or even a more classical MLP with ReLU or sigmoid activations, is not expected to represent sine and cosine functions over unbounded inputs. The reason for this limitation lies in the universal approximation results for feedforward networks. These results generally guarantee approximation within compact convergence on bounded sets, such as in the compactification of $\mathbb{R}$ or in $L^p$ spaces, as described in \citet{cybenko1989approximation}, \citet{ito1992approximation}, and \citet{arora2016understanding}. None of these results extend to uniform approximation of sine or cosine over the entire real line.

Recent work by \citet{van2024noncompact} addresses the universal approximation capabilities in the space $C_b(\mathbb{R})$, which is the class of bounded continuous functions over $\mathbb{R}$. This result is particularly relevant since approximating sine and cosine functions uniformly over $\mathbb{R}$ would fall under this category. According to their Proposition 5.5, sine and cosine functions cannot be uniformly approximated using certain activation functions, limiting the feasibility of approximating $f(x) = \frac{e^{i \pi x} + 1}{2}$ in a typical MLP architecture.

Thus, it is unrealistic to expect a typical MLP, with ReLU or sigmoid activations, to implement the function $f(x) = \frac{e^{i \pi x} + 1}{2}$ uniformly for arbitrarily large inputs. Consequently, a construction based on such a function would either necessitate custom activation functions, such as periodic activations specifically designed to handle sine and cosine, or require the size of the model to scale with the input length. Either solution removes apparent contradiction with Theorem \ref{thm:parity}, as these adjustments fall outside the scope of the assumptions made in our proof.

\citet{wang2024state} and \citet{orvieto2024universality} provide universal approximation guarantees for SSMs, but these guarantees depend on the size of the approximating network growing with input length. This dependency is clearly stated in Proposition 3.6 and Proposition 3.9 of \citet{wang2024state}, and further emphasized by \citet{orvieto2024universality}. in their Remark 2. Our results, in contrast, pertain to the existence of a single SSM capable of recognizing a formal language for any input length, independent of network size. Thus, such universal approximation results do not undermine Theorem \ref{thm:parity}.
\end{remark}

\subsection{Proof of Theorem~\ref{thm:regular}}\label{sec:acc:star-free}

Our proof of Theorem~\ref{thm:regular} will rely on the algebraic theory of finite automata, specifically the cascade product and the Krohn-Rhodes Theorem \citep{krohn1965algebraic}. These techniques, originally developed in the 1960s, have recently been introduced to the theoretical study of transformers by \citet{liu2022transformers}; we provide self-contained definitions and somewhat different notation, tailored to our proofs about state-space models. In general, we will find that the properties of state-space models allow more natural and directly length-generalizing implementations of these algebraic notions than what is possible for transfomers.

Recall the definition of a finite-state-automaton (Definition~\ref{def:automaton}).
Our construction will build on an important operation on automata, the cascade product \citep{krohn1965algebraic, eilenberg1974automata, ginzburg2014algebraic}:
\begin{defin}
Given two automata $\mathcal{A}_1, \mathcal{A}_2$ with associated alphabets $\Sigma_1, \Sigma_2$ and state sets $Q_1, Q_2$ such that
\begin{equation}
\Sigma_2 = Q_1 \times \Sigma_1,
\end{equation}
the cascade product $A_2 \wreath A_1$ is the automaton given by
\begin{itemize}
    \item $\Sigma = \Sigma_1$
    \item $Q = Q_2 \times Q_1$
    \item $q_0$ is the tuple of the starting states of $\mathcal{A}_2, \mathcal{A}_1$
    \item $u(\langle q,p \rangle,\sigma) = \left\langle u_2\left(q, \langle p, \sigma\rangle\right), u_1(p, \sigma)\right\rangle$
\end{itemize}
\end{defin}

We note that the literature usually uses ``$\circ$'' for the cascade product \citep[e.g.][]{eilenberg1974automata}. To avoid collision with the elementwise product ``$\circ$'' (e.g., (\ref{eq:recurrence-ssm})), we here instead use ``$\wr$'', usually used for the wreath product -- a product on monoids with an effect analogous to the cascade product \citep{almeida1995finite}.

While the formal definition is cumbersome, the intuition behind it is simple:
The cascade product corresponds to first reading a word ${\bf w}$ with $\mathcal{A}_1$, recording the state sequence $q_0, q_1, \dots, q_{|{\bf w}|} \in Q_1$ and -- at each $t=1,\dots,|{\bf w}|$ -- pasting the state $q_{t-1}$ together with the input symbol $w_t \in \Sigma_1$ -- resulting in a word over a new alphabet $Q_1\times\Sigma_1$, and then running $A_2$ on the resulting word.
The overall state of $A_2 \wreath A_1$ after reading a word is the tuple of the states reached by $A_2$ and $A_1$.
Note that we write $A_2 \wreath A_1$, rather than, $A_1 \wreath A_2$, because the second argument of the cascade product ($A_1$) intuitively reads the input first, preprocessing it for the other automaton, $A_2$ -- the cascade product can thus be viewed as a kind of function composition.

The somewhat inscrutable update rule for $u(\cdot,\cdot)$ encodes the action of $\mathcal{A}_1$ in the second component, and the action of $\mathcal{A}_2$ on the extended alphabet in the first component.
There is a close analogy to the stacking of sequence models, and we will leverage this analogy to translate cascade products into multilayer SSMs. The fundamental background here is the following classical fact:
\begin{fact}
    [Consequence of Krohn-Rhodes Theorem \citep{krohn1965algebraic} and Sch{\"u}tzenberger's Theorem \citep{schutzenberger1965finite}]
    Each star-free regular language is recognized by an iterated cascade product of set-reset automata, $(\dots(\mathcal{A}_1 \wreath \dots)\wreath \mathcal{A}_{n-1}) \wreath \mathcal{A}_n$, where each $\mathcal{A}_i$ is a set-reset automaton.
\end{fact}

This result follows from the Krohn-Rhodes decomposition theorem \citep{krohn1965algebraic}, which states that any finite-state automaton can be expressed as an iterated cascade product of simple automata, specifically finite simple groups and reset automata. Moreover, Sch{\"u}tzenberger's Theorem \citep{schutzenberger1965finite} characterizes star-free regular languages as those whose syntactic monoids are aperiodic, meaning they contain no nontrivial groups. Therefore, the decomposition for star-free languages involves only set-reset automata, leading to the stated cascade product structure. We now formally show that cascade products can be translated to SSM stacking.
We need an auxiliary lemma, which provides a single-layer SSM that encodes the input $w_{t-1}$ in state $h_t$ -- we will use it to forward information about the state of $\mathcal{A}_1$ at $t-1$ to $\mathcal{A}_2$ at $t$:
\begin{lemma}\label{thm:readout-last}
Let $\Sigma$ be an alphabet, and consider words $w\in\Sigma^*$.
There is a one-layer SSM with $d=4|\Sigma|$ such that, for $t=2,\dots,|w|$, the character $w_{t-1}$ can be read out from $z_t$ at finite precision.
\end{lemma}

To prove Lemma~\ref{thm:readout-last}, a first idea is to use an exponential moving average with $A=1/2$ to encode the recent input characters in $h_t$; this effectively encodes the full history into the binary expansion of $h_t$, and in particular  allows reading out the second-last input in principle.
However, such a construction does not work at finite precision, because rounding may make it impossible to  extract even the second-most-significant bit.\footnote{Informally, in binary, 0.0111111...111 and 0.1 are arbitrarily close.} 
We avoid this problem simply by taking $A=1/4$, effectively utilizing only every two digits in the binary expansion of $h_t$, ensuring that the second-last input can be read out at a constant margin.
We now provide the formal proof:
\begin{proof}
We begin by showing the claim in the special case $\Sigma = \{1,0\}$.
Here, we take $d=4$, and
\begin{align*}
h_0 &= [0,0,0,0]^T \\
A(e_0) &= [1/4, 1/4, 0, 0]^T \\
A(e_1) &= [1/4, 1/4, 0, 0]^T \\
B(e_0) &= [1, 0, 1, 0]^T \\
B(e_1) &= [0, 1, 0, 1]^T 
\end{align*}
Now we separately consider the state $h_t$ depending on the form of the prefix $w_{1\dots t}$ (here $w_{1\dots t}$ refers to first $t$ characters in the word). If $w_{1\dots t} = \dots 00$ (the last 2 characters of the prefix are $00$), then
\begin{equation}\label{eq:bounds-for-ht-recursion}
    h_t = \left(\begin{matrix} \in [1,2] \\ \in [0,1/8] \\ 1 \\ 0 \end{matrix}\right)
\end{equation}
because
\begin{align*}
    h_t = & A(e_0) \circ h_{t-1} + B(e_0)  \\
    = & A(e_0) \circ  A(e_0) \circ h_{t-2} + A(e_0) \circ  B(e_0) + B(e_0)  \\
    = & [1/16, 1/16, 0, 0]^T \circ h_{t-2} + [1/16, 1/16, 0, 0]^T \circ [1, 0, 1, 0]^T + [1, 0, 1, 0]^T  \\
    = & \left(\begin{matrix} 
    \frac{1}{16} (h_{t-2})_1 + \frac{1}{16} + 1 \\ \frac{1}{16} (h_{t-2})_2\\ 1 \\ 0  \end{matrix}\right)
    \\
\end{align*}
By definition of $A$ and $B$, each entry in $h_{t-2}$ is in $[0,2]$; the claim (\ref{eq:bounds-for-ht-recursion}) then follows.
If $w_{1\dots t} = \dots 10$, then, by a similar calculation
\begin{equation}
    h_t = \left(\begin{matrix} \in [1,1.25] \\ \in [1/4,1/2] \\ 1 \\ 0 \end{matrix}\right)
\end{equation}
In particular, assuming $w_t=0$, one can read off $w_{t-1}$  from $(h_t)_2$ with a margin of size 1/8.
As $w_t$ is encoded in $h_t$ and due to symmetry, analogous statements hold when $w_t=1$.

Now, for each $\sigma\in\Sigma$, we run such a one-layer SSM where $0$ represents $\sigma$ and $1$ represents all other characters.\footnote{In fact, using a binary encoding of $\Sigma$, one can achieve $d = 4\log |\Sigma|$.}
% What does running in parallel mean ? 
By running these in parallel (i.e. executing these operations with the same SSM layer simultaneously, utilising the width of the SSM layer) we obtain an SSM with $d=4|\Sigma|$ from whose states one can read out $w_{t-1}$ at finite precision. As the entries in $h_t$ are all bounded by $2$, we find $\|h_t\|_2 \leq 2\sqrt{d}$ independent of $t$, and the margin is still bounded away from zero after normalization, and thus in $z_t$, where we can assume $\operatorname{Mix}_1$, $\operatorname{Mix}_2$ to be the identity.
\end{proof}

\begin{remark}\label{app:local-co v}
Some SSMs include local convolutions \citep[e.g.][]{fu2023hungry, Gu2023Mamba} or local attention \citep{De2024Griffin}, which aggregate information from a local window of some width $\Delta>0$. These do not increase the expressive capacity beyond SSMs as we have defined in (\ref{eq:recurrence-ssm}-\ref{eq:mix}), as aggregation of local information can be simulated with a single SSM layer:
Using the layer constructed in the proof of Lemma~\ref{thm:readout-last}, given the state $h_t$, once one has read out $w_{t-1}$ as described in the proof, one can recover $h_{t-1}$ from $h_t$ and $x_t$; then inductively read out $w_{t-2}$ using $h_{t-1}$ and $x_{t-1}$, etc.
Thus, up to any given width $\Delta>0$, one can read out $w_{t-\Delta}, \dots, w_{t-1}$ from the state $h_t$ of this layer at finite precision.
\end{remark}

We are now ready to translate cascade products into SSM stacking:
\begin{lemma}\label{thm:cascade-ssm}
    Let $\mathcal{A}_1$, $\mathcal{A}_2$ be two finite-state-automata, and assume that there are two SSMs with top-level states $z^{(L_1,1)}$ and $z^{(L_2,2)}$ that map each ${\bf w}$ to the state sequences under $\mathcal{A}_1$, $\mathcal{A}_2$, at finite precision.
    
    Formally, on a word ${\bf w}$, $\rho_1(z^{(L_1,1)}_t)$ and $\rho_2(z^{(L_2,2)}_t)$ provide the state sequences of $\mathcal{A}_1$, $\mathcal{A}_2$.

    Then there is an SSM with $L_1+L_2+1$ layers that maps each ${\bf w}$ to the state sequence under $\mathcal{A}_2 \wreath \mathcal{A}_2$, again at finite precision.

\end{lemma}
We note that a conceptually related result holds for transformers \citep[Lemma 12 in][]{liu2022transformers}. However, SSMs allow a simpler and length-independent construction, as they do not require positional encodings to implement such a construction. 

\begin{proof}
The lower layers are based on the SSM modeling $\mathcal{A}_1$.
We duplicate each channel, so we now have $2d$ dimensions.
We further add $d$ further dimensions that directly pass on the input embeddings, i.e., $A\equiv 0$, $B \equiv 1$, $\operatorname{Mix}_j \equiv Id$ on these dimensions.

In the resulting SSM, $z^{L_1}_t$ indicates both $w_t$ itself, and the state reached by $\mathcal{A}_1$ after reading $w_{1\dots t}$. The state is redundantly indicated by two separate sets of $d$ dimensions; the character $w_t$ is indicated by $d$ further state.

Note, however, that the second automaton in the cascade product requires access to the state $q_{t-1}$ rather than $q_t$.

For this, we add a layer provided by Lemma \ref{thm:readout-last}, of width $4|Q|$.
Additional $2d$ dimensions pass on (1) $w_t$, and (2) the state that $\mathcal{A}_1$ reaches after reading the prefix $w_{1\dots t}$.

We now have $L_1+1$ layers where $z^{L_1+1}_t$ has $2d+4|Q|$ dimensions and indicates (1) $w_t$, (2) the state that $\mathcal{A}_1$ reaches after reading the prefix $w_{1\dots t}$, (3) the state that $\mathcal{A}_1$ reaches after reading the prefix $w_{1\dots t-1}$.

The first and third piece of information are now fed into the second SSM; the second piece is passed on in $d$ additional dimensions.
As we allowed $A$ and $B$ to be arbitrary functions, we redefine these in the lowest layer of that second SSM to read out from the $4|Q|$-dimensional component indicating (3), providing the desired second-to-last state. 

We have constructed an SSM with $L_1+L_2+1$ layers, where $z^{L_1+L_2+1}_t$ indicates (1) $w_t$, (2) the state that $\mathcal{A}_1$ reaches after reading the prefix $w_{1\dots t}$, (3) the state  that $\mathcal{A}_2$ reaches after reading the prefix $w_{1\dots t}$ pasted with the state sequence of $\mathcal{A}_1$.
This information is sufficient for reading out the state sequence of $\mathcal{A}_2 \wreath \mathcal{A}_1$.

Note that the number of channels may not be consistent, as it is $3d$ in the top and bottom parts, but $2d+4|Q|$ in the middle; we simply pad to the larger dimensionality.
\end{proof}

We are now ready to show the existence of length-generalizing SSMs for any star-free state tracking problem, and conclude with the theorem:
\begin{thm}[Restated from Theorem~\ref{thm:regular}]
Let $\mathcal{L}$ be a regular language.
The following are equivalent:
\begin{enumerate}
    \item There is an SSM satisfying \textsc{nonnegative} that predictively models $\mathcal{L}$ at all input lengths, at finite precision
    \item $\mathcal{L}$ is star-free.
\end{enumerate}
\end{thm}
\begin{proof}
We need to show:
\begin{enumerate}
\item SSMs at finite precision can predictively model all star-free languages. For each language, a single SSMs is applicable at arbitrary lengths.

\item Assuming \textsc{Nonnegative}, finite-precision SSMs cannot recognize any non-star-free regular language.
\end{enumerate}
The second statement is Corollary~\ref{eq:nonnegative-lower-bound-periodic}; it suffices to prove the first statement.

Assume $\mathcal{L}$ is star-free.
By the Krohn-Rhodes theorem, there is an automaton $\mathcal{A}$ that is a cascade product of some set-reset automata that recognizes $\mathcal{L}$.
By Lemmas~\ref{thm:set-reset} and \ref{thm:cascade-ssm}, there is an SSM that computes the state sequence of that automaton.

Now we note that, since $\mathcal{A}$ recognizes $\mathcal{L}$, the state $q$ after reading ${\bf w}$ is sufficient for determining the set of characters that can follow this prefix in any element of $\mathcal{L}$.
For, assume otherwise, then there are words ${\bf w}$, ${\bf w'}$ such that $u(q_0, {\bf w}) = u(q_0, {\bf w'})$ and $\sigma \in \Sigma$ such that ${\bf w}\sigma\Sigma^* \cap \mathcal{L} \neq \emptyset$ but ${\bf w'}\sigma\Sigma^* \cap \mathcal{L} = \emptyset$; then $u(q_0, {\bf w}\sigma) = u(q_0, {\bf w'}\sigma)$ but the set $R$ (\ref{eq:recog-set}) is reachable from $u(q_0, {\bf w}\sigma)$ but not $u(q_0, {\bf w'}\sigma)$, contradiction.

Hence, the SSM's outputs can be transformed, by composing $\rho$ with a map from states to next-character sets, to predictively model $\mathcal{L}$.
\end{proof}

\begin{thm}\label{thm:solvable}
SSMs with complex-valued coefficients evading both \textsc{nonnegative} and \textsc{time-invariant} can represent all regular languages known to be in $\operatorname{TC}^0$.
\end{thm}
    We we do not use this theorem in the main paper, due to the nonexistence (as far as we know) of implemented SSMs with this property.
\begin{proof}
SSMs evading both \textsc{nonnegative} and \textsc{time-invariant} can count modulo any integer $k$, using $d=1$ and $A(e_1) = e^{2\pi i/k}$, $A(e_0) = 1$, $B \equiv 0$, $h_0=1$.
This is a generalization of the construction for PARITY described in Section~\ref{sec:app:parity}, since $e^{2\pi i/2}=-1$.

    The set of regular languages known to be in $\operatorname{TC}^0$ is the set of regular languages whose syntactic monoid contains no non-solvable groups \citep{barrington1992regular}.
    These languages are recognized by cascade products of set-reset automata and automata perfoming modular counting \citep{straubing2012finite}.
    By the remark above, together with Lemma~\ref{thm:set-reset} and Lemma~\ref{thm:cascade-ssm}, such cascade products can be simulated by SSMs.
\end{proof}

\subsection{Maintaining Counters}\label{sec:app:counting}

As the first step in showing Theorem~\ref{thm:counter-languages}, we show that SSMs can maintain unbounded counters, and that one can read out the values of such counters, up to finite bounds, even at finite precision:
\begin{lemma}\label{thm:counter_app}
Let $C > 0$ be an integer.
Let any function $u : \Sigma \rightarrow \mathbb{Z}^C$ be given.
  Let $L \in \mathbb{N}$.
  Then a one-layer SSM with finite precision can compute, at each position $i=1,\dots,T$:
  \begin{equation}
  \operatorname{max}\left(\operatorname{min}\left(    \sum_{j=1}^i u(w_i), L\right), -L\right)
  \end{equation}
  in the sense that $\rho$ can read this out from $z_i^{(1)}$ with finite precision.
\end{lemma}

\begin{proof}
Define $d = 2L+1$.
Define $h_0 = {\bf 0} \in \mathbb{R}^d$.
For each $x \in \Sigma$, define $A(e_x) = {\bf 1} \in \mathbb{R}^d$ and $B(e_x)_i \in \mathbb{R}^d$ by $B(e_x)_i = u(x)$.
In order to read out the state $h_t$ up to a limit $L$, we define
    \begin{equation}
        \phi(h_t, x_t) = \operatorname{Norm}(h_t + [0, 1, -1, 2, -2, \dots, -L, L])
    \end{equation}
    By testing which entries of the result are negative or positive, one can read out the state up to $L$ even after rounding $\phi(h_t, x_t)$ to finite precision.
    The proof straightforwardly extends to multiple counters.
\end{proof}

We are ready to prove the theorem:
\begin{thm}[Restated from Theorem~\ref{thm:counter-languages}]\label{cor:counter-languages_app}
The languages Dyck-1, Shuffle-Dyck, n-ary Boolean Expressions, \(a^nb^n\), \(a^nb^nc^n\), and \(a^nb^nc^nd^n\), (defined in Appendix~\ref{sec:language-definitions}) can each be predictively modeled by an SSM.
\end{thm} 
\begin{proof}
    For each of these languages, we first define an assignment $u : \Sigma\rightarrow\mathbb{Z}^C$:
    \begin{align*}
    \text{For } \text{$a^nb^n$:} & \text{ (here, $C$=1)}\\
        u(a) &= 1 \\
        u(b) &= -1 \\
    \text{For } \text{Dyck-1:} & \text{ (here, $C$=1)}\\
        u(``(") &= 1 \\
        u(``)") &= -1 \\
        \text{For Shuffle-Dyck-$k$} & \text{ (here, $C=k$)} \\
        u(``(_i") &= (0, \dots, 0, 1, 0\dots 0)\ \ \ \  \text{ where $1$ is in the $i$-th slot} \\
        u(``)_i") &= (0, \dots, 0, -1, 0 \dots 0)\ \ \ \  \text{ where $-1$ is in the $i$-th slot} \\        
     \text{For }   a^nb^nc^n:  & \text{ (here, $C$=2)}\\
        u(a) &= (1,0) \\
        u(b) &= (-1,1) \\
        u(c) &= (0,-1) \\
      \text{For }  a^nb^nc^nd^n:  & \text{ (here, $C$=3)}\\
        u(a) &= (1,0,0) \\
        u(b) &= (-1,1,0) \\
        u(c) &= (0,-1,1) \\
        u(d) &= (0,-1,-1) \\
      \text{For }   \text{Boolean Expressions:} & \text{ (here, $C$=1)}\\
        u(\langle VALUE\rangle) &= -1 \\
        u(\langle n-ARY\rangle) &= +n
    \end{align*}
For each of these mappings, we use Lemma~\ref{thm:counter_app} at $L=1$ to construct a one-layer SSMs that can, for each of the $C$ counters, distinguish the values $\leq -1, 0, \geq 1$.

In parallel, we pass on the input symbol itself in $\log |\Sigma|$ further dimensions.

Overall, the output $z_t$ of single SSM layer provides, at every position, both the original symbol in $\Sigma$ and an element of $\{\leq -1, 0 \geq 1\}^C$.

We can thus view the output of this layer as a string over an enriched string of symbols $\sigma_1 \times \sigma_2 \in \Sigma \times \{\leq -1, 0 \geq 1\}^C$. Based on this, one can predictively model these languages as follows.

For Dyck-1, the next token is EOS or ``('' if $\sigma_2=0$, and ``('' or ``)'' after any other prefix (note that predictive modeling assumes valid prefixes).

Shuffle-$k$-Dyck is similar: EOS is allowed if and only if all counters are zero. An opening bracket is always allowed. A closing bracket is only allowed if the respective counter is $>0$.

For $a^nb^n$, the next token is $a$ or $b$ if $\sigma_1=a$; $b$ if $\sigma=(a,\geq 1)$ or $(b,\geq 1)$; EOS if $\sigma=(b,0)$.

Constructions for $a^nb^nc^n$, $a^nb^nc^nd^n$ are similar.

For Boolean expressions, the next token is $\langle n-ARY\rangle$ or EOS if $\sigma_2=0$, and any other token otherwise.

All of these constructions can be encoded using an appropriate function $\rho$ applying to $z_t$.    
\end{proof}

\subsection{Bounded-Depth Dyck}\label{app:sec:bounded-dyck}

\begin{defin}
The language $Dyck_{K,h}$ \citep{hewitt2020rnns, yao-etal-2021-self} is given by the CFG with the nonterminals $\{S_0, S_1, \dots, S_{h-1}, S_h\}$ and the following production rules:
\begin{align*}
    S_h \rightarrow& (_1 S_{h-1} )_1 | \dots | (_K S_{h-1} )_K | \epsilon\\
    S_{h-1} \rightarrow& (_1 S_{h-2} )_1 | \dots | (_K S_{h-2} )_K | \epsilon\\
    \dots & \dots \\
    S_2 \rightarrow& (_1 S_{1} )_1 | \dots | (_K S_{1} )_K | \epsilon \\
    S_1 \rightarrow& (_1 S_{0} )_1 | \dots | (_K S_{0} )_K | \epsilon \\
    S_0 \rightarrow& \epsilon
\end{align*}
and the start symbol $S_h$.
\end{defin}

\begin{thm}[Restated from Theorem~\ref{thm:bounded-dyck}]\label{thm:bounded-dyck_app}
    There is a two-layer SSM with $d = \mathcal{O}(h \log K)$ that predictively models $Dyck_{K,h}$ at all input lengths, at finite precision.
\end{thm}

\begin{proof}
In the first layer, we calculate each token depth up to $h$ using Lemma~\ref{thm:counter_app}.
After the first layer, at each position, the activations will indicate both the depth up to $h$, and the identity of the symbol.
The space of activations is thus $\{0, \dots, h\} \times \{ (_1, )_1, \dots, (_K, )_K\}$.
We then, for each depth $l=1, \dots, h$, define a set-reset automaton (Definition~\ref{def:set-reset}) given by the set $Q_l := \{l\} \times \{ (_1, )_1, \dots, (_K, )_K\}$.
Running all of these set-reset automata will tell us, for each depth, the identity of the last bracket at that depth.
We can deduce the maximum depth $h'$ at which the last bracket is an opening one, and thus infer the set of valid next symbols.
The activity of these set-reset automata can, in parallel, be simulated by a second SSM layer using Lemma~\ref{thm:set-reset}.
We need $h$ such automata, and each SSM has width $\log K$.
\end{proof}

\section{Definitions of Languages}
\label{sec:language-definitions}

Here, we provide formal definitions of languages from the test suite based on \citet{bhattamishra2020ability}. Descriptions follow \citet{bhattamishra2020ability}, and are included here for self-containment.
In all cases, our data generation setup is directly taken from \citep{bhattamishra2020ability}.

\subsection{Regular Languages}
\textbf{Tomita Grammars.} Used primarily as a benchmark language family for assessing sequence to sequence models \citep{tomita1982dynamic}, some of the languages in this family are star-free (with dot-depth of 1) and some non-star-free. All the regular languages of the family are defined on the alphabet \(\Sigma=\{0, 1\}\). Individual language definitions are available in Table \ref{tab:tomitas}.

\(\boldsymbol{D_n}\). We follow the definition of \cite{bhattamishra2020ability} to define the \(D_n\) family of star-free languages. In our experiments, we only generate \(D_2\), \(D_3\), \(D_4\), and \(D_{12}\) languages; \(D_1\) is equivalent to Tomita-2. %, given the proven equivalence of Tomita-2 and \(D_2\). 
All the languages of the family are defined on the alphabet of \(\Sigma=\{a,b\}\).
\(D_n = (aD_{n-1}b)^*\) has level $n$ in the dot-depth hierarchy.

\textbf{PARITY.} PARITY is the set of all strings on the alphabet \(\Sigma=\{0,1\}\) such that the number of 1's is even. This language can be easily recognized by a DFA with just two states.  

\textbf{Others.}  We further have the non-star-free languages \((aa)^*\), \((aaaa)^*\) and \((abab)^*\), and the star-free languages \(aa^*bb^*cc^*dd^*ee^*\), \(\{ab\}^*d\{b,c\}^*\), and \(\{0,1,2\}^*02^*\).

\subsection{Counter Languages}\label{sec:language-defin-counter}
\textbf{Dyck and Shuffle-Dyck.}  Dyck-1 is defined on the alphabet \(\Sigma=\{[,]\}\) and derived using the following CFG production rule:
\begin{math}
S \rightarrow ( S ) | SS | \epsilon
\end{math}.

We further use the family of Shuffle-k languages \citep{shuffle-dyck-def}.
Shuffle-Dyck-k is defined in terms of $\Sigma = \{(_1, )_1, \dots, (_k, )_k\}$.
It is defined as the shuffle of $k$ Dyck-1 languages, each defined in terms of the alphabet $\Sigma_i=\{(_i, )_i\}$ where $i=1,\dots,k$.

\textbf{\textit{n}-ary Boolean Expressions.} 
This is the set of valid expressions over various operators.
We focus on up-to-3-ary expressions, defined using the following grammar:

\begin{math}
S \rightarrow \langle \mathrm{VALUE} \rangle \\
S \rightarrow \langle \mathrm{UNARY\ OPERATOR} \rangle\ S \\
S \rightarrow \langle \mathrm{BINARY\ OPERATOR} \rangle\ S\ S \\
S \rightarrow \langle \mathrm{TERNARY\ OPERATOR} \rangle\ S\ S\ S
\end{math}

This language is recognized by a counter automaton \citep{Fischer1968}.

\paragraph{Others}
We further include the languages of the forms \(a^nb^n\), \(a^nb^nc^n\), and \(a^nb^nc^nd^n\).

\begin{table}[ht]
\small
    \centering
    \begin{tabularx}{\columnwidth}{p{1.5cm}|p{1.5cm}|p{8cm}}
    \hline
         Grammar & Star-Free & Definition\\
         \hline
         1 & Yes & 1* \\
         2 & Yes & (10)* \\
         3 & No & strings without \(1^{2n+1}0^{2m+1}\) substrings\\
         4 & Yes & strings without any 000's substrings\\
         5 & No & strings of even length with an even number of 1's \\
         6 & No & strings where number of 0's - number of 1's is divisible by 3\\
         7 & Yes & 0*1*0*1\\
    \end{tabularx}
    \caption{Tomita Grammars}
    \label{tab:tomitas}
\end{table}

% \clearpage

\begin{table*}[ht]
\centering
\begin{tabular}{|m{2cm}|m{2cm}|m{2cm}|m{2.5cm}|m{2.5cm}|}
\hline
\textbf{Language} & \textbf{Model} & \textbf{Bin-1[1, 50]} & \textbf{Bin-2[51, 100]} & \textbf{Bin-3[101, 150]} \\ \hline
\multicolumn{1}{|c|}{\multirow{4}{*}{Dyck-1}}
 & Transformer & 100.0 & 100.0 & 100.0 \\ \cline{2-5} 
 & Mamba1 & 100.0 & 62.6 & 13.91 \\ \cline{2-5} 
 & Mamba2 & 100.0 & 49.1 & 9.5 \\ \cline{2-5} 
 & Mamba3 & 100.0 & 53.95 & 10.0 \\ \hline
\multicolumn{1}{|c|}{\multirow{4}{*}{Shuffle-2}}
 & Transformer & 100.0 & 100.0 & 93.0 \\ \cline{2-5} 
 & Mamba1 & 100.0 & 49.5 & 2.3 \\ \cline{2-5} 
 & Mamba2 & 100.0 & 61.5 & 8.2 \\ \cline{2-5} 
 & Mamba3 & 100.0 & 65.5 & 9.7 \\ \hline 
\multicolumn{1}{|c|}{\multirow{4}{*}{Shuffle-4}}
 & Transformer & 100.0 & 100.0 & 98.8 \\ \cline{2-5} 
 & Mamba1 & 100.0 & 44.4 & 4.3 \\ \cline{2-5} 
 & Mamba2 & 100.0 & 63.8 & 7.2 \\ \cline{2-5} 
 & Mamba3 & 100.0 & 56.2 & 7.8 \\ \hline 
\multicolumn{1}{|c|}{\multirow{4}{*}{Shuffle-6}}
 & Transformer & 100.0 & 99.9 & 94.0 \\ \cline{2-5} 
 & Mamba1 & 100.0 & 39.4 & 3.4 \\ \cline{2-5} 
 & Mamba2 & 100.0 & 61.2 & 6.75 \\ \cline{2-5} 
 & Mamba3 & 100.0 & 59.6 & 9.85 \\ \hline 
\multicolumn{1}{|c|}{\multirow{4}{*}{Boolean-3}}
 & Transformer & 100.0 & 100.0 & 99.8 \\ \cline{2-5} 
 & Mamba1 & 99.75 & 65.7 & 7.05 \\ \cline{2-5} 
 & Mamba2 & 99.95 & 47.25 & 2.3 \\ \cline{2-5} 
 & Mamba3 & 100.0 & 73.45 & 8.6 \\ \hline 
 \multicolumn{1}{|c|}{\multirow{4}{*}{Boolean-5}}
 & Transformer & 100.0 & 99.8 & 99.0 \\ \cline{2-5} 
 & Mamba1 & 99.9 & 30.05 & 7.6 \\ \cline{2-5} 
 & Mamba2 & 100.0 & 80.2 & 14.9 \\ \cline{2-5} 
 & Mamba3 & 99.25 & 60.7 & 6.25 \\ \hline 
\multicolumn{1}{|c|}{\multirow{4}{*}{\(a^nb^n\)}}
 & Transformer & 100.0 & 100.0 & 100.0 \\ \cline{2-5} 
 & Mamba1 & 100.0 & 4.1 & 0 \\ \cline{2-5} 
 & Mamba2 & 100.0 & 9.4 & 0 \\ \cline{2-5} 
 & Mamba3 & 100.0 & 21.3 & 0 \\ \hline
\multicolumn{1}{|c|}{\multirow{4}{*}{\(a^nb^nc^n\)}}
 & Transformer & 100.0 & 100.0 & 100.0 \\ \cline{2-5} 
 & Mamba1 & 100.0 & 0 & 0 \\ \cline{2-5} 
 & Mamba2 & 100.0 & 7.6 & 0 \\ \cline{2-5} 
 & Mamba3 & 100.0 & 5.1 & 0 \\ \cline{2-5} \hline
\multicolumn{1}{|c|}{\multirow{4}{*}{\(a^nb^nc^nd^n\)}}
 & Transformer & 100.0 & 100.0 & 99.4 \\ \cline{2-5} 
 & Mamba1 & 100.0 & 4.76 & 0 \\ \cline{2-5} 
 & Mamba2 & 100.0 & 0 & 0 \\ \cline{2-5} 
 & Mamba3 & 100.0 & 0 & 0 \\ \cline{2-5} \hline
\end{tabular}
\caption{Accuracies on the counter Languages from the \citet{bhattamishra2020ability} test suite. Transformer results reported based on \citet{bhattamishra2020ability}. For Mamba, we report best settings (chosen based on inputs of length [1,50]) at 1 (Mamba1), 2 (Mamba2), 3 (Mamba3) layers. Results for the best-performing layer count, from the first two bins, are shown in Figure~\ref{fig:results-bhattamishra}. On these languages, there is also a third bin.}
\end{table*}

\begin{table*}[ht]
\centering
\begin{tabular}{|m{2cm}|m{3.5cm}|m{2cm}|m{2.5cm}|}
\hline
\textbf{Language} & \textbf{Model} & \textbf{Bin-1[1, 50]} & \textbf{Bin-2[51, 100]} \\ \hline
\multicolumn{1}{|c|}{\multirow{4}{*}{Tomita 1}}
 & Transformer & 100.0 & 100.0  \\ \cline{2-4} 
 & Mamba1 & 100.0 & 100.0  \\ \cline{2-4} 
 & Mamba2 & 100.0 & 100.0 \\ \cline{2-4} 
 & Mamba3 & 100.0 & 100.0 \\ \hline
\multicolumn{1}{|c|}{\multirow{4}{*}{Tomita 4}}
 & Transformer & 100.0 & 92.4 \\ \cline{2-4} 
 & Mamba1 & 100.0 & 100.0 \\ \cline{2-4} 
 & Mamba2 & 100.0 & 100.0 \\ \cline{2-4} 
 & Mamba3 & 100.0 & 100.0 \\ \hline 
\multicolumn{1}{|c|}{\multirow{4}{*}{Tomita 7}}
 & Transformer & 100.0 & 100.0 \\ \cline{2-4} 
 & Mamba1 & 100.0 & 100.0 \\ \cline{2-4} 
 & Mamba2 & 100.0 & 100.0 \\ \cline{2-4} 
 & Mamba3 & 100.0 & 100.0 \\ \hline 
\multicolumn{1}{|c|}{\multirow{4}{*}{Tomita 2}}
 & Transformer & 100.0 & 100.0 \\ \cline{2-4} 
 & Mamba1 & 100.0 & 100.0 \\ \cline{2-4} 
 & Mamba2 & 100.0 & 100.0 \\ \cline{2-4} 
 & Mamba3 & 100.0 & 100.0 \\ \hline 
\multicolumn{1}{|c|}{\multirow{4}{*}{\(aa^*bb^*cc^*dd^*ee^*\)}}
 & Transformer & 100.0 & 100.0 \\ \cline{2-4} 
 & Mamba1 & 100.0 & 100.0 \\ \cline{2-4} 
 & Mamba2 & 100.0 & 100.0 \\ \cline{2-4} 
 & Mamba3 & 100.0 & 100.0 \\ \hline 
 \multicolumn{1}{|c|}{\multirow{4}{*}{\(\{a, b\}^*d\{b,c\}^*\)}}
 & Transformer & 100.0 & 100.0 \\ \cline{2-4} 
 & Mamba1 & 100.0 & 100.0 \\ \cline{2-4} 
 & Mamba2 & 100.0 & 100.0 \\ \cline{2-4} 
 & Mamba3 & 100.0 & 100.0 \\ \hline 
\multicolumn{1}{|c|}{\multirow{4}{*}{\(\{0, 1, 2\}^*02^*\)}}
 & Transformer & 100.0 & 68.7 \\ \cline{2-4} 
 & Mamba1 & 100.0 & 100.0 \\ \cline{2-4} 
 & Mamba2 & 100.0 & 100.0 \\ \cline{2-4} 
 & Mamba3 & 100.0 & 100.0 \\ \hline
\multicolumn{1}{|c|}{\multirow{4}{*}{\(D_2\)}}
 & Transformer & 74.6 & 3.1 \\ \cline{2-4} 
 & Mamba1 & 100.0 & 100.0 \\ \cline{2-4} 
 & Mamba2 & 100.0 & 100.0 \\ \cline{2-4} 
 & Mamba3 & 100.0 & 100.0 \\ \cline{2-4} \hline
\multicolumn{1}{|c|}{\multirow{4}{*}{\(D_3\)}}
 & Transformer & 80.9 & 8.5 \\ \cline{2-4} 
 & Mamba1 & 100.0 & 100.0 \\ \cline{2-4} 
 & Mamba2 & 100.0 & 100.0 \\ \cline{2-4} 
 & Mamba3 & 100.0 & 100.0 \\ \cline{2-4} \hline
\multicolumn{1}{|c|}{\multirow{4}{*}{\(D_4\)}}
 & Transformer & 90.2 & 3.3 \\ \cline{2-4} 
 & Mamba1 & 100.0 & 100.0 \\ \cline{2-4} 
 & Mamba2 & 100.0 & 100.0 \\ \cline{2-4} 
 & Mamba3 & 100.0 & 100.0 \\ \cline{2-4} \hline
\multicolumn{1}{|c|}{\multirow{4}{*}{\(D_{12}\)}}
 & Transformer & 95.18 & 1.5 \\ \cline{2-4} 
 & Mamba1 & 93.65 & 93.35 \\ \cline{2-4} 
 & Mamba2 & 99.9 & 95.55 \\ \cline{2-4} 
 & Mamba3 & 99.99 & 99.85 \\ \cline{2-4} \hline 
\end{tabular}
\caption{Accuracies on the regular Languages from the \citet{bhattamishra2020ability} test suite - 1st half. Transformer results reported based on \citet{bhattamishra2020ability}. For Mamba, we report best settings (chosen based on inputs of length [1,50]) at 1 (Mamba1), 2 (Mamba2), 3 (Mamba3) layers. Results for the best-performing layer count are also shown in Figure~\ref{fig:results-bhattamishra}.}
\end{table*}

\begin{table*}[ht]
\centering
\begin{tabular}{|m{2cm}|m{3.5cm}|m{2cm}|m{2.5cm}|}
\hline
\textbf{Language} & \textbf{Model} & \textbf{Bin-1[1, 50]} & \textbf{Bin-2[51, 100]} \\ \hline 
\multicolumn{1}{|c|}{\multirow{4}{*}{Parity}}
 & Transformer & 68.7 & 0 \\ \cline{2-4} 
 & Mamba1 & 26.95 & 0 \\ \cline{2-4} 
 & Mamba2 & 80.05 & 4.15 \\ \cline{2-4} 
 & Mamba3 & 91.15 & 16.7 \\ \cline{2-4} \hline 
\multicolumn{1}{|c|}{\multirow{4}{*}{\((aa)^*\)}}
 & Transformer & 100.0 & 0 \\ \cline{2-4} 
 & Mamba1 & 2.1 & 0 \\ \cline{2-4} 
 & Mamba2 & 2.1 & 0 \\ \cline{2-4} 
 & Mamba3 & 4.2 & 0 \\ \cline{2-4} \hline 
\multicolumn{1}{|c|}{\multirow{4}{*}{\((aaaa)^*\)}}
 & Transformer & 100.0 & 0 \\ \cline{2-4} 
 & Mamba1 & 0 & 0 \\ \cline{2-4} 
 & Mamba2 & 0 & 0 \\ \cline{2-4} 
 & Mamba3 & 4.0 & 0 \\ \cline{2-4} \hline  
\multicolumn{1}{|c|}{\multirow{4}{*}{\((abab)^*\)}}
 & Transformer & 100.0 & 2.5 \\ \cline{2-4} 
 & Mamba1 & 0 & 0 \\ \cline{2-4} 
 & Mamba2 & 0 & 0 \\ \cline{2-4} 
 & Mamba3 & 0 & 0 \\ \cline{2-4} \hline  
\multicolumn{1}{|c|}{\multirow{4}{*}{Tomita 3}}
 & Transformer & 75.4 & 10.8 \\ \cline{2-4} 
 & Mamba1 & 25.99 & 12.49 \\ \cline{2-4} 
 & Mamba2 & 36.88 & 17.05 \\ \cline{2-4} 
 & Mamba3 & 60.85 & 29.37 \\ \cline{2-4} \hline  
\multicolumn{1}{|c|}{\multirow{4}{*}{Tomita 5}}
 & Transformer & 29.3 & 0.0 \\ \cline{2-4} 
 & Mamba1 & 15.94 & 0 \\ \cline{2-4} 
 & Mamba2 & 34.5 & 0 \\ \cline{2-4} 
 & Mamba3 & 38.4 & 0 \\ \cline{2-4} \hline 
\multicolumn{1}{|c|}{\multirow{4}{*}{Tomita 6}}
 & Transformer & 88.8 & 0 \\ \cline{2-4} 
 & Mamba1 & 7.2 & 0 \\ \cline{2-4} 
 & Mamba2 & 37.8 & 0 \\ \cline{2-4} 
 & Mamba3 & 54.56 & 0.04 \\ \cline{2-4} \hline    
\end{tabular}
\caption{Accuracies on the regular Languages from the \citet{bhattamishra2020ability} test suite - continued. Transformer results reported based on \citet{bhattamishra2020ability}. For Mamba, we report best settings (chosen based on inputs of length [1,50]) at 1 (Mamba1), 2 (Mamba2), 3 (Mamba3) layers.  Results for the best-performing layer count are also shown in Figure~\ref{fig:results-bhattamishra}.}
\end{table*}

\section{Experimental Details}
All experiments used the Mamba reference implementation\footnote{\url{https://github.com/state-spaces/mamba/blob/main/README.md}}.
xUnless stated otherwise, we followed the defaults given there ( $d_{state}=16$, $d_{conv}=4$, 
    $expand=2$), as we found the default combination to work better than other options. We tuned $d_{model}$ for each language.

\subsection{Test Suite from \texorpdfstring{\citet{bhattamishra2020ability}}{Bhattamishra et al 2020}} \label{sec:app:setup-bhattamishra}

\paragraph{Data Preparation}
For all the languages, we use either the data prepared by \citet{bhattamishra2020ability} or---where not available---their data-generation scripts, allowing full comparability with results they reported for transformers. 
We used their official code and data release at \url{https://github.com/satwik77/Transformer-Formal-Languages} (last commit 48eea2e; MIT license).
Training sets typically consist of 10K samples, with lengths varying between 1 to 50.
There are two heldout bins: one with in-distribution lengths ([1,50]), and one testing length generalization (lengths [51,100]).
The first one was used for hyperparameter optimization.
Each bin typically contains around 2K samples.
However for languages such as $a^nb^n$, where the number of positive examples in each bin was limited, all possible examples for that bin are included.

\paragraph{Hyperparameters}
For each language, we conducted extensive hyperparameter search. We varied the $d_{model}$ parameter in Mamba across the set \{16, 32, 64, 128, 256\}. Additionally, we experimented with the number of layers in our model, ranging from 1 to 3, training each configuration for 100 epochs. For languages where Mamba performed well, this number of layers was sufficient. However, for languages where Mamba struggled, we increased the number of layers up to 12, with little to no success.

We used the AdamW optimizer.
To identify optimal learning rates, we started with a coarse hyperparameter search using values from the set \{0.001, 0.0001, 0.00001\}. If one of these learning rates showed high performance, we conducted a more fine-grained search to find the optimal learning rate. Finally, we varied the batch size from \{16, 32, 64\} for datasets with 10K training examples. For languages like $a^nb^n$ with limited training size, we searched for an optimal batch size within the set \{5, 10\}.

\subsection{FlipFlop}\label{sec:app:setup-flipFlop}

We obtained the dataset of \citet{Liu2023FlipFlop} from their release, \url{ https://huggingface.co/datasets/synthseq/flipflop} (MIT license). Our setup corresponds to the deterministic (``clean'') mode in \citet{Liu2023FlipFlop}. Matching Figure 2 in \citet{Liu2023FlipFlop}, we evaluated both with in-distribution data (matching the distribution of the training dataset) with $p_i=0.8, p_w=0.1, p_r=0.1$, and using an out of distribution sparse tail with $p_i=0.98, p_w=0.01, p_r=0.01$, where $p_i, p_w, p_r$ refer to the probabilities of that instruction appearing in input sequences.

We trained a one-layer Mamba with the default parameters\footnote{From \url{https://github.com/state-spaces/mamba/blob/main/README.md}}, setting $d_{model}$ to 16 with the AdamW optimizer using a learning rate of $3x10^{-4}$ and a batch size of 16.

Following the evaluation criteria for LSTMs in \citet{Liu2023FlipFlop}, we compute the test every 100 training steps on our validation sets of choice, by randomly sampling around $10^3$ samples from each set in every evaluation cycle.

\subsection{Bounded Hierarchical Structure}\label{sec:experimental-dyck}
We built on the official code and data release of \citet{yao-etal-2021-self} at \url{https://github.com/princeton-nlp/dyck-transformer} (last commit: 5d21fcf).
We train a 2-layer Mamba and a 1-layer Mamba on $Dyck_{K,h}$ with $K=8$ and $h=10$. The training set and the validation set contains samples of lengths $\le 700$, while the test set contains samples of lengths $700 \le n \le 1400$. We train Mamba with a varying number of layers $l \in \{1, 2\}$ and $d_{model} \in \{20, 30, 40, 50, 60, 70, 80, 90, 100\}$.
We use the Adam optimizer with an initial learning rate of 0.01 or 0.001, using cross-entropy loss.
After training for 100 epochs (with early stopping allowed in case of convergence), we select the learning rate with the better training performance.

\begin{figure}
    \centering
    \begin{minipage}[b]{0.32\textwidth}
        \includegraphics[width=\textwidth]{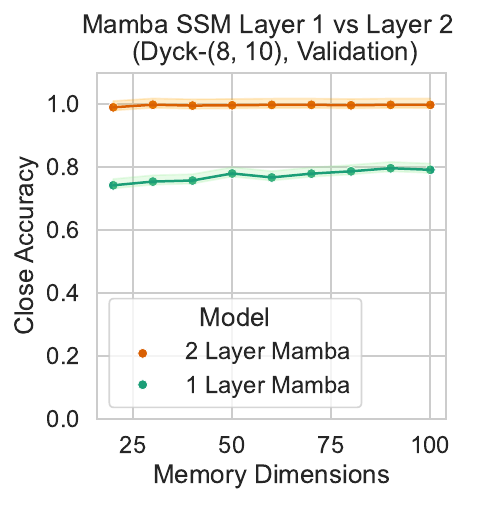}
        \label{fig:bdyck_val}
    \end{minipage}
    \begin{minipage}[b]{0.32\textwidth}
        \includegraphics[width=\textwidth]{images/test_close_acc_comparison.pdf}
        \label{fig:bdyck_val3}
    \end{minipage}    
    \caption{Mamba Accuracy on  $Dyck_{8,10}$, on the development set (length $\leq 700$, same length range as training set) and test set (length $700 \leq n \leq 1400$). The latter is also plotted in Figure~\ref{fig:dyck}.}
    \label{fig:three_figures}
\end{figure}

\section{Finite Precision Assumption}\label{app:precision}
As described in Section~\ref{sec:background}, we adopt the \emph{finite precision} notion used by \citet{weiss2018practical}:
We allow an unbounded number of integer bits, but only $p$ fractional bits, where $p$ is a sufficiently large constant (e.g., $p=8$), independent of the length of the input.

There are a variety of related precision notions in the theoretical literature on neural sequence models -- here, we discuss the effect of other notions on our results:
\begin{enumerate}
\item \textbf{Infinite precision} Infinite precision allows any parameter and intermediate value to be an arbitrary number. Such a setting is unrealistic, as it would allow encoding arbitrary detail about the input into infinite precision  \citep[e.g.][]{siegelmann2012neural} and read these out with sufficiently powerful functions ($A$, $B$, $\phi$) in (\ref{eq:recog-set}) -- this would lead to the unrealistic conclusion that any function and language could be represented. For this reason, theoretical work has often adopted restricted precision notions.

\item \textbf{Finite inventory of values}, where integer and fractional bits are both restricted. 
Such a setup may be justified based on the fact that any real computer has bounded memory, though such a setup precludes \emph{any} positive results on non-finite-state problems for \emph{any} computational architecture.\footnote{For instance, a Turing machine with bounded memory and thus a bounded tape is equivalent to a finite-state automaton.}

Such a restrictive setup would not affect our positive results on Flip-Flop, Star-Free, and bounded-depth Dyck languages (Theorems~\ref{thm:flip_flop}, \ref{thm:regular}, \ref{thm:bounded-dyck}), as these all use \emph{bounded} finite-precision activation values. As this is a \emph{more} restricted setup than the one we are assuming, this also would not affect our negative results about PARITY and non-star-free languages (Theorems~\ref{thm:parity}, \ref{thm:regular}). These results are thus highly robust to variations of the finite precision assumption.

Such a more restrictive definition would, however, mean that, for unbounded counting (Theorem~\ref{thm:counter-languages}), modeling is only possible up to a bound determined by the number of possible values---this is the one place where our results would be impacted. Indeed, we do observe that Mamba learns these counter languages on training lengths but struggles with length generalization. Transformers, on the other hand, can represent these languages with bounded activations (due to the constructions in \cite{bhattamishra2020ability}), and show strong length generalization.

An intermediary between infinite and finite precision is notions of precision where the number of allowed bits slowly increases with the input length, e.g., logarithmically. Such a setup has particularly been adopted for transformers \citep{merrill2023logic}, because a finite-precision assumption leads to very low expressivity in transformers. For SSMs, on the other hand, we find that finite precision assumptions are sufficient for showing a broad range of positive results.

\end{enumerate}

\newpage
\section*{NeurIPS Paper Checklist}

%%% BEGIN INSTRUCTIONS %%%

%%% END INSTRUCTIONS %%%

\begin{enumerate}

\item {\bf Claims}
    \item[] Question: Do the main claims made in the abstract and introduction accurately reflect the paper's contributions and scope?
    \item[] Answer: \answerYes{} % Replace by \answerYes{}, \answerNo{}, or \answerNA{}.
    \item[] Justification: The abstract and introduction summarise the theoretical and empirical results. We took to only include well-supported claims.
    \item[] Guidelines:
    \begin{itemize}
        \item The answer NA means that the abstract and introduction do not include the claims made in the paper.
        \item The abstract and/or introduction should clearly state the claims made, including the contributions made in the paper and important assumptions and limitations. A No or NA answer to this question will not be perceived well by the reviewers. 
        \item The claims made should match theoretical and experimental results, and reflect how much the results can be expected to generalize to other settings. 
        \item It is fine to include aspirational goals as motivation as long as it is clear that these goals are not attained by the paper. 
    \end{itemize}

\item {\bf Limitations}
    \item[] Question: Does the paper discuss the limitations of the work performed by the authors?
    \item[] Answer: \answerYes{} % Replace by \answerYes{}, \answerNo{}, or \answerNA{}.
    \item[] Justification: The Discussion section includes a paragraph on Limitations.
    \item[] Guidelines:
    \begin{itemize}
        \item The answer NA means that the paper has no limitation while the answer No means that the paper has limitations, but those are not discussed in the paper. 
        \item The authors are encouraged to create a separate "Limitations" section in their paper.
        \item The paper should point out any strong assumptions and how robust the results are to violations of these assumptions (e.g., independence assumptions, noiseless settings, model well-specification, asymptotic approximations only holding locally). The authors should reflect on how these assumptions might be violated in practice and what the implications would be.
        \item The authors should reflect on the scope of the claims made, e.g., if the approach was only tested on a few datasets or with a few runs. In general, empirical results often depend on implicit assumptions, which should be articulated.
        \item The authors should reflect on the factors that influence the performance of the approach. For example, a facial recognition algorithm may perform poorly when image resolution is low or images are taken in low lighting. Or a speech-to-text system might not be used reliably to provide closed captions for online lectures because it fails to handle technical jargon.
        \item The authors should discuss the computational efficiency of the proposed algorithms and how they scale with dataset size.
        \item If applicable, the authors should discuss possible limitations of their approach to address problems of privacy and fairness.
        \item While the authors might fear that complete honesty about limitations might be used by reviewers as grounds for rejection, a worse outcome might be that reviewers discover limitations that aren't acknowledged in the paper. The authors should use their best judgment and recognize that individual actions in favor of transparency play an important role in developing norms that preserve the integrity of the community. Reviewers will be specifically instructed to not penalize honesty concerning limitations.
    \end{itemize}

\item {\bf Theory Assumptions and Proofs}
    \item[] Question: For each theoretical result, does the paper provide the full set of assumptions and a complete (and correct) proof?
    \item[] Answer: \answerYes{} % Replace by \answerYes{}, \answerNo{}, or \answerNA{}.
    \item[] Justification: Every theorem is proven formally in the appendix. Assumptions about the SSM architecture are stated  formally.
    \item[] Guidelines:
    \begin{itemize}
        \item The answer NA means that the paper does not include theoretical results. 
        \item All the theorems, formulas, and proofs in the paper should be numbered and cross-referenced.
        \item All assumptions should be clearly stated or referenced in the statement of any theorems.
        \item The proofs can either appear in the main paper or the supplemental material, but if they appear in the supplemental material, the authors are encouraged to provide a short proof sketch to provide intuition. 
        \item Inversely, any informal proof provided in the core of the paper should be complemented by formal proofs provided in appendix or supplemental material.
        \item Theorems and Lemmas that the proof relies upon should be properly referenced. 
    \end{itemize}

    \item {\bf Experimental Result Reproducibility}
    \item[] Question: Does the paper fully disclose all the information needed to reproduce the main experimental results of the paper to the extent that it affects the main claims and/or conclusions of the paper (regardless of whether the code and data are provided or not)?
    \item[] Answer: \answerYes{} % Replace by \answerYes{}, \answerNo{}, or \answerNA{}.
    \item[] Justification: Besides describing experimental details, we have publicly released our codebase on GitHub and added a link to it in the main paper with instructions on how to run our experiments, ensuring that our findings can be reproduced. 
    \item[] Guidelines:
    \begin{itemize}
        \item The answer NA means that the paper does not include experiments.
        \item If the paper includes experiments, a No answer to this question will not be perceived well by the reviewers: Making the paper reproducible is important, regardless of whether the code and data are provided or not.
        \item If the contribution is a dataset and/or model, the authors should describe the steps taken to make their results reproducible or verifiable. 
        \item Depending on the contribution, reproducibility can be accomplished in various ways. For example, if the contribution is a novel architecture, describing the architecture fully might suffice, or if the contribution is a specific model and empirical evaluation, it may be necessary to either make it possible for others to replicate the model with the same dataset, or provide access to the model. In general. releasing code and data is often one good way to accomplish this, but reproducibility can also be provided via detailed instructions for how to replicate the results, access to a hosted model (e.g., in the case of a large language model), releasing of a model checkpoint, or other means that are appropriate to the research performed.
        \item While NeurIPS does not require releasing code, the conference does require all submissions to provide some reasonable avenue for reproducibility, which may depend on the nature of the contribution. For example
        \begin{enumerate}
            \item If the contribution is primarily a new algorithm, the paper should make it clear how to reproduce that algorithm.
            \item If the contribution is primarily a new model architecture, the paper should describe the architecture clearly and fully.
            \item If the contribution is a new model (e.g., a large language model), then there should either be a way to access this model for reproducing the results or a way to reproduce the model (e.g., with an open-source dataset or instructions for how to construct the dataset).
            \item We recognize that reproducibility may be tricky in some cases, in which case authors are welcome to describe the particular way they provide for reproducibility. In the case of closed-source models, it may be that access to the model is limited in some way (e.g., to registered users), but it should be possible for other researchers to have some path to reproducing or verifying the results.
        \end{enumerate}
    \end{itemize}

\item {\bf Open access to data and code}
    \item[] Question: Does the paper provide open access to the data and code, with sufficient instructions to faithfully reproduce the main experimental results, as described in supplemental material?
    \item[] Answer: \answerYes{} % Replace by \answerYes{}, \answerNo{}, or \answerNA{}.
    \item[] Yes, we have publicly released our codebase on GitHub and added a link to it in the main paper with instructions on how to run our experiments. 
    \item[] Guidelines:
    \begin{itemize}
        \item The answer NA means that paper does not include experiments requiring code.
        \item Please see the NeurIPS code and data submission guidelines (\url{https://nips.cc/public/guides/CodeSubmissionPolicy}) for more details.
        \item While we encourage the release of code and data, we understand that this might not be possible, so “No” is an acceptable answer. Papers cannot be rejected simply for not including code, unless this is central to the contribution (e.g., for a new open-source benchmark).
        \item The instructions should contain the exact command and environment needed to run to reproduce the results. See the NeurIPS code and data submission guidelines (\url{https://nips.cc/public/guides/CodeSubmissionPolicy}) for more details.
        \item The authors should provide instructions on data access and preparation, including how to access the raw data, preprocessed data, intermediate data, and generated data, etc.
        \item The authors should provide scripts to reproduce all experimental results for the new proposed method and baselines. If only a subset of experiments are reproducible, they should state which ones are omitted from the script and why.
        \item At submission time, to preserve anonymity, the authors should release anonymized versions (if applicable).
        \item Providing as much information as possible in supplemental material (appended to the paper) is recommended, but including URLs to data and code is permitted.
    \end{itemize}

\item {\bf Experimental Setting/Details}
    \item[] Question: Does the paper specify all the training and test details (e.g., data splits, hyperparameters, how they were chosen, type of optimizer, etc.) necessary to understand the results?
    \item[] Answer: \answerYes{} % Replace by \answerYes{}, \answerNo{}, or \answerNA{}.
    \item[] Justification: We have specified the training and test details in a dedicated appendix section. Our experiments use freely available dataset or dataset generating scripts from previous studies, which we cite. Our uploaded code base ensures that all settings are reproducible.
    \item[] Guidelines:
    \begin{itemize}
        \item The answer NA means that the paper does not include experiments.
        \item The experimental setting should be presented in the core of the paper to a level of detail that is necessary to appreciate the results and make sense of them.
        \item The full details can be provided either with the code, in appendix, or as supplemental material.
    \end{itemize}

\item {\bf Experiment Statistical Significance}
    \item[] Question: Does the paper report error bars suitably and correctly defined or other appropriate information about the statistical significance of the experiments?
    \item[] Answer: \answerNo{} % Replace by \answerYes{}, \answerNo{}, or \answerNA{}.
    \item[] Justification: In the empirical results relevant to supporting our positive theoretical results, the accuracy attained by Mamba is either very close to 100\%  or far from 100\%. As the test sets include hundreds or thousands of data points, error bars would not add substantial further information.
    \item[] Guidelines:
    \begin{itemize}
        \item The answer NA means that the paper does not include experiments.
        \item The authors should answer "Yes" if the results are accompanied by error bars, confidence intervals, or statistical significance tests, at least for the experiments that support the main claims of the paper.
        \item The factors of variability that the error bars are capturing should be clearly stated (for example, train/test split, initialization, random drawing of some parameter, or overall run with given experimental conditions).
        \item The method for calculating the error bars should be explained (closed form formula, call to a library function, bootstrap, etc.)
        \item The assumptions made should be given (e.g., Normally distributed errors).
        \item It should be clear whether the error bar is the standard deviation or the standard error of the mean.
        \item It is OK to report 1-sigma error bars, but one should state it. The authors should preferably report a 2-sigma error bar than state that they have a 96\% CI, if the hypothesis of Normality of errors is not verified.
        \item For asymmetric distributions, the authors should be careful not to show in tables or figures symmetric error bars that would yield results that are out of range (e.g. negative error rates).
        \item If error bars are reported in tables or plots, The authors should explain in the text how they were calculated and reference the corresponding figures or tables in the text.
    \end{itemize}

\item {\bf Experiments Compute Resources}
    \item[] Question: For each experiment, does the paper provide sufficient information on the computer resources (type of compute workers, memory, time of execution) needed to reproduce the experiments?
    \item[] Answer: \answerYes{} % Replace by \answerYes{}, \answerNo{}, or \answerNA{}.
    \item[] Justification: We have provided information in the appendix section on experimental setup.
    \item[] Guidelines:
    \begin{itemize}
        \item The answer NA means that the paper does not include experiments.
        \item The paper should indicate the type of compute workers CPU or GPU, internal cluster, or cloud provider, including relevant memory and storage.
        \item The paper should provide the amount of compute required for each of the individual experimental runs as well as estimate the total compute. 
        \item The paper should disclose whether the full research project required more compute than the experiments reported in the paper (e.g., preliminary or failed experiments that didn't make it into the paper). 
    \end{itemize}
    
\item {\bf Code Of Ethics}
    \item[] Question: Does the research conducted in the paper conform, in every respect, with the NeurIPS Code of Ethics \url{https://neurips.cc/public/EthicsGuidelines}?
    \item[] Answer: \answerYes{} % Replace by \answerYes{}, \answerNo{}, or \answerNA{}.
    \item[] Justification: We have verified compliance with the NeurIPS Code of Ethics,
    \item[] Guidelines:
    \begin{itemize}
        \item The answer NA means that the authors have not reviewed the NeurIPS Code of Ethics.
        \item If the authors answer No, they should explain the special circumstances that require a deviation from the Code of Ethics.
        \item The authors should make sure to preserve anonymity (e.g., if there is a special consideration due to laws or regulations in their jurisdiction).
    \end{itemize}

\item {\bf Broader Impacts}
    \item[] Question: Does the paper discuss both potential positive societal impacts and negative societal impacts of the work performed?
    \item[] Answer: \answerNA{} % Replace by \answerYes{}, \answerNo{}, or \answerNA{}.
    \item[] Justification: The paper studies foundational properties of neural network architectures, and we foresee no positive or negative societal impact of the results.
    \item[] Guidelines:
    \begin{itemize}
        \item The answer NA means that there is no societal impact of the work performed.
        \item If the authors answer NA or No, they should explain why their work has no societal impact or why the paper does not address societal impact.
        \item Examples of negative societal impacts include potential malicious or unintended uses (e.g., disinformation, generating fake profiles, surveillance), fairness considerations (e.g., deployment of technologies that could make decisions that unfairly impact specific groups), privacy considerations, and security considerations.
        \item The conference expects that many papers will be foundational research and not tied to particular applications, let alone deployments. However, if there is a direct path to any negative applications, the authors should point it out. For example, it is legitimate to point out that an improvement in the quality of generative models could be used to generate deepfakes for disinformation. On the other hand, it is not needed to point out that a generic algorithm for optimizing neural networks could enable people to train models that generate Deepfakes faster.
        \item The authors should consider possible harms that could arise when the technology is being used as intended and functioning correctly, harms that could arise when the technology is being used as intended but gives incorrect results, and harms following from (intentional or unintentional) misuse of the technology.
        \item If there are negative societal impacts, the authors could also discuss possible mitigation strategies (e.g., gated release of models, providing defenses in addition to attacks, mechanisms for monitoring misuse, mechanisms to monitor how a system learns from feedback over time, improving the efficiency and accessibility of ML).
    \end{itemize}
    
\item {\bf Safeguards}
    \item[] Question: Does the paper describe safeguards that have been put in place for responsible release of data or models that have a high risk for misuse (e.g., pretrained language models, image generators, or scraped datasets)?
    \item[] Answer: \answerNA{} % Replace by \answerYes{}, \answerNo{}, or \answerNA{}.
    \item[] Justification: We are not releasing data or models, and foresee no risk of misues of our theoretical results.
    \item[] Guidelines:
    \begin{itemize}
        \item The answer NA means that the paper poses no such risks.
        \item Released models that have a high risk for misuse or dual-use should be released with necessary safeguards to allow for controlled use of the model, for example by requiring that users adhere to usage guidelines or restrictions to access the model or implementing safety filters. 
        \item Datasets that have been scraped from the Internet could pose safety risks. The authors should describe how they avoided releasing unsafe images.
        \item We recognize that providing effective safeguards is challenging, and many papers do not require this, but we encourage authors to take this into account and make a best faith effort.
    \end{itemize}

\item {\bf Licenses for existing assets}
    \item[] Question: Are the creators or original owners of assets (e.g., code, data, models), used in the paper, properly credited and are the license and terms of use explicitly mentioned and properly respected?
    \item[] Answer: \answerYes{} % Replace by \answerYes{}, \answerNo{}, or \answerNA{}.
    \item[] Justification: The paper cites the original papers, and provides details in the Appendix section on experimental details.
    \item[] Guidelines:
    \begin{itemize}
        \item The answer NA means that the paper does not use existing assets.
        \item The authors should cite the original paper that produced the code package or dataset.
        \item The authors should state which version of the asset is used and, if possible, include a URL.
        \item The name of the license (e.g., CC-BY 4.0) should be included for each asset.
        \item For scraped data from a particular source (e.g., website), the copyright and terms of service of that source should be provided.
        \item If assets are released, the license, copyright information, and terms of use in the package should be provided. For popular datasets, \url{paperswithcode.com/datasets} has curated licenses for some datasets. Their licensing guide can help determine the license of a dataset.
        \item For existing datasets that are re-packaged, both the original license and the license of the derived asset (if it has changed) should be provided.
        \item If this information is not available online, the authors are encouraged to reach out to the asset's creators.
    \end{itemize}

\item {\bf New Assets}
    \item[] Question: Are new assets introduced in the paper well documented and is the documentation provided alongside the assets?
    \item[] Answer: \answerNA{} % Replace by \answerYes{}, \answerNo{}, or \answerNA{}.
    \item[] Justification: The paper releases no new assets.
    \item[] Guidelines:
    \begin{itemize}
        \item The answer NA means that the paper does not release new assets.
        \item Researchers should communicate the details of the dataset/code/model as part of their submissions via structured templates. This includes details about training, license, limitations, etc. 
        \item The paper should discuss whether and how consent was obtained from people whose asset is used.
        \item At submission time, remember to anonymize your assets (if applicable). You can either create an anonymized URL or include an anonymized zip file.
    \end{itemize}

\item {\bf Crowdsourcing and Research with Human Subjects}
    \item[] Question: For crowdsourcing experiments and research with human subjects, does the paper include the full text of instructions given to participants and screenshots, if applicable, as well as details about compensation (if any)? 
    \item[] Answer: \answerNA{} % Replace by \answerYes{}, \answerNo{}, or \answerNA{}.
    \item[] Justification: The paper involves no crowdsourcing nor human subjects research.
    \item[] Guidelines:
    \begin{itemize}
        \item The answer NA means that the paper does not involve crowdsourcing nor research with human subjects.
        \item Including this information in the supplemental material is fine, but if the main contribution of the paper involves human subjects, then as much detail as possible should be included in the main paper. 
        \item According to the NeurIPS Code of Ethics, workers involved in data collection, curation, or other labor should be paid at least the minimum wage in the country of the data collector. 
    \end{itemize}

\item {\bf Institutional Review Board (IRB) Approvals or Equivalent for Research with Human Subjects}
    \item[] Question: Does the paper describe potential risks incurred by study participants, whether such risks were disclosed to the subjects, and whether Institutional Review Board (IRB) approvals (or an equivalent approval/review based on the requirements of your country or institution) were obtained?
    \item[] Answer: \answerNA{} % Replace by \answerYes{}, \answerNo{}, or \answerNA{}.
    \item[] Justification: The paper involves no crowdsourcing nor human subjects research.
    \item[] Guidelines:
    \begin{itemize}
        \item The answer NA means that the paper does not involve crowdsourcing nor research with human subjects.
        \item Depending on the country in which research is conducted, IRB approval (or equivalent) may be required for any human subjects research. If you obtained IRB approval, you should clearly state this in the paper. 
        \item We recognize that the procedures for this may vary significantly between institutions and locations, and we expect authors to adhere to the NeurIPS Code of Ethics and the guidelines for their institution. 
        \item For initial submissions, do not include any information that would break anonymity (if applicable), such as the institution conducting the review.
    \end{itemize}

\end{enumerate}

\end{document}